\documentclass[11pt]{article}
\usepackage[margin=1in]{geometry}
\usepackage{amsmath,amssymb,amsthm,mathtools,bm,booktabs,array,microtype,float}
\usepackage{tikz}
\usetikzlibrary{arrows.meta,positioning,cd}
\usepackage[colorlinks=true,allcolors=blue!55!black]{hyperref}

\newtheorem{definition}{Definition}[section]
\newtheorem{assumption}[definition]{Assumption}
\newtheorem{theorem}[definition]{Theorem}
\newtheorem{proposition}[definition]{Proposition}
\newtheorem{corollary}[definition]{Corollary}

\newtheorem{remark}[definition]{Remark}
\newcommand{\Prob}{\mathcal P}
\newcommand{\Ctxt}{\mathcal C}
\newcommand{\Words}{\mathcal W}
\newcommand{\Trans}{\mathfrak T}
\newcommand{\Bmap}{\mathbf B}
\newcommand{\Kern}{\mathsf K}
\newcommand{\Dlex}{\mathsf D}
\newcommand{\Id}{\mathrm I}

\newcommand{\Exp}{\mathsf E}

\title{Foundations of Stochastic Lexical Calculus\\
\large Semantic Descent and Random Dynamics on Probability Simplices}
\author{Matthew F. Dixon\\Artificial Intelligence Finance Institute (AIFI)\\
\texttt{matthew.dixon@aifi.edu}}
\date{Working manuscript, July 2026}

\begin{document}
\maketitle

\begin{abstract}
Large language models produce prompt-dependent probabilities over words,
whereas scientific systems require uncertainty over meaningful states that
can be updated as evidence arrives.  We develop an observable framework for
determining when language-derived probabilities support such a sequential
state representation.  Theoretically, we define typed measurable
transformations of contextual language, construct a minimal closed
representation, and give necessary and sufficient conditions for semantic
updates to exist uniquely.  We bound irreducible nonclosure and accumulated
error, and under average contraction prove existence, uniqueness and
stability of an external random recursion on a probability simplex.  These
results define a stochastic lexical calculus without attributing an internal
calculus to the language model.  Empirically, frozen experiments test the
observable implications.  Raw prompt-conditioned probabilities fail the
prespecified invariance gate; after prompt-specific calibration, a common
three-state representation passes the stability gates and covers $28$ of
$30$ untouched eight-step paths, or $0.933$ at nominal level $0.90$.
Accordingly, language probabilities support a stochastic state only
conditionally on verified closure, stability and coverage within a declared
operating domain.
\end{abstract}

\section{The problem}
Calculus begins by specifying objects, admissible changes, observables, and laws relating local change to composition and accumulation. Natural language has the same need. A word occurrence can be replaced, a qualification inserted, a clause negated, two passages concatenated, or a collection of claims reordered. These operations act on contextual expressions rather than on isolated dictionary entries, and their order may alter meaning.

Existing linguistic calculi primarily formalize grammatical derivability, typed composition, denotation, or symbolic rewriting. Discrete calculus supplies general difference operators, but it does not determine which transformations preserve linguistic information, which contextual distinctions matter, or when a semantic representation remains closed under subsequent changes. We address that missing layer.

This gap has become operational rather than merely terminological. A fitted language model assigns probabilities to lexical continuations, while users reason about coarser meanings such as factual alternatives, risk orientations, diagnoses, intentions, or confidence levels. Recent work measures uncertainty by grouping semantically equivalent generations, asking models for confidence, calibrating candidate-token probabilities, or training linguistically calibrated responses \cite{band2024,farquhar2024,kadavath2022,kuhn2023,kumar2024,tian2023}. These methods establish that language-derived uncertainty can be useful, but they leave a prior mathematical question unresolved: when does a probability-valued summary of language retain exactly the distinctions required by later information transformations? Without such closure, successive summaries need not form a state process, even when each summary is individually well calibrated.

The primitive object in this paper is contextual language. Probability laws, model outputs, and declared semantic states are examples of lexical observables. This ordering matters: the calculus is not defined by a particular language model, tokenization, ontology, or Bayesian posterior. Those enter later as representations on which the general laws can be tested.

The motivating question is simple. If two contexts are treated as the same state, must every relevant linguistic intervention affect them in the same way? If not, the state has discarded transformation-relevant lexical information. No closed recursion on that state is justified. This obstruction, and the canonical representation that removes it, organize the theory.

For example, suppose two market reports are both summarized by the state
probabilities $(0.6,0.3,0.1)$ for risk-on, mixed and risk-off.  In the first
report, the favorable assessment is driven by falling inflation; in the
second, it is driven by improving corporate earnings.  Appending the same new
observation---``inflation unexpectedly rises''---may change the first
assessment much more than the second.  The identical present state therefore
does not determine the next state.  A recursion using only those three
probabilities has omitted the distinction needed for the update.

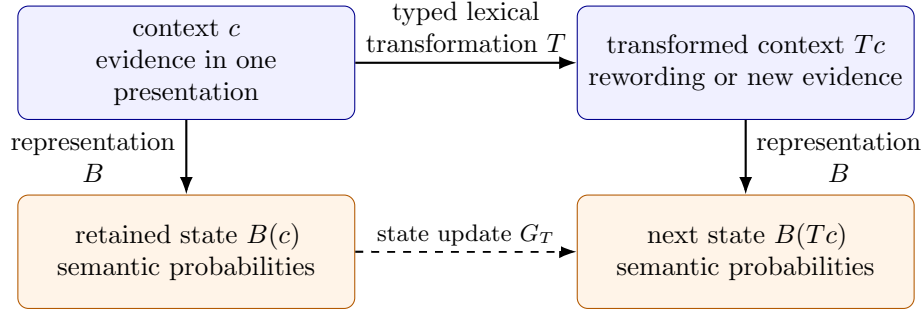
\begin{figure}[t]
\centering
\begin{tikzpicture}[
  >=Latex,
  context/.style={draw=blue!55!black,fill=blue!6,rounded corners,
    align=center,text width=42mm,minimum height=15mm,font=\small},
  state/.style={draw=orange!70!black,fill=orange!9,rounded corners,
    align=center,text width=42mm,minimum height=15mm,font=\small},
  lab/.style={font=\small,align=center}
]
\node[context] (c) at (0,2.5) {context $c$\\evidence in one presentation};
\node[context] (tc) at (7.4,2.5) {transformed context $Tc$\\rewording or new evidence};
\node[state] (b) at (0,0) {retained state $B(c)$\\semantic probabilities};
\node[state] (bt) at (7.4,0) {next state $B(Tc)$\\semantic probabilities};
\draw[->,thick] (c) -- node[above,lab]{typed lexical\\transformation $T$} (tc);
\draw[->,thick] (c) -- node[left,lab]{representation\\$B$} (b);
\draw[->,thick] (tc) -- node[right,lab]{representation\\$B$} (bt);
\draw[->,thick,dashed] (b) -- node[above,font=\footnotesize,fill=white,inner sep=2pt]
{state update $G_T$} (bt);
\end{tikzpicture}
\caption{The central closure problem.  The upper path transforms the full
context and then measures its retained semantic state.  A recursion on the
lower state space is legitimate precisely when the dashed map is
representative independent, so that $B(Tc)=G_T\{B(c)\}$.  Rewordings that
preserve the declared information should leave the state stable; genuine new
evidence may move it, but must do so through the retained state if that state
is to support autonomous sequential modelling.}
\label{fig:closure-problem}
\end{figure}

Figure~\ref{fig:closure-problem} separates two operations that are often
conflated.  A presentation change alters the words while preserving the
declared information; an evidence update changes the information itself.  A
useful semantic state should be stable to the former and responsive to the
latter.  More strongly, if it is to be updated recursively, its response to
every admitted transformation must be determined by the retained state rather
than by lexical detail that the representation has already discarded.

Why does this lead to a \emph{stochastic lexical calculus}?  In a deployed
system, neither evidence nor the linguistic transformations through which it
is presented arrive as a fixed deterministic sequence.  Reports, questions,
qualifications and corrections arrive over time; their order and content are
uncertain; and each may alter a probability distribution over the declared
semantic states.  For three states, for example, the evolving quantity is a
point $(p_1,p_2,p_3)$ on the probability simplex, not a single response.
Successive lexical transformations therefore generate a random path on that
simplex.

Calling this path a stochastic state process requires more than observing
that its coordinates move.  The same retained state must imply the same
successor state under a declared transformation; otherwise the apparent
recursion still depends on lexical information that has been discarded.  Once
this closure condition holds, finite lexical differences describe one-step
movement, cocycle laws describe composition, and contraction and perturbation
bounds determine whether approximation error dissipates or accumulates.
These are the components of the stochastic lexical calculus developed here.

The terminology is intentionally discrete and representation based.  We do
not begin by imposing a continuous-time diffusion, stochastic integral or
It\^o formula on language.  Rather, we derive a random dynamical system from
typed contextual transformations and prove when it descends to a stable
simplex-valued recursion.  Continuous-time limits, when scientifically
appropriate, would come only after this elementary closure problem has been
solved.

\subsection{Contributions and boundary of the claim}

The term \emph{calculus} has long been used in linguistics and logic. Lambek calculus formalizes grammatical composition; lambda calculi support formal semantics; differential lambda calculus differentiates computational terms; Brzozowski derivatives act on formal languages; and discrete calculus supplies finite-difference laws. Earlier works have also used the phrases \emph{calculus of words}, \emph{lexical calculi}, and \emph{calculus of language}. We therefore make no claim to have introduced mathematical reasoning about words.

\textbf{The focal contribution is structural and stochastic, not merely
lexical.}  The problem lies at the intersection of several established
mathematical structures in computer science: typed semantics specifies legal
language transformations, partial-map categories represent operations that
are not defined on every context, behavioral equivalence determines which
contexts may share a state, and probability kernels describe random
evolution.  The deterministic transformation algebra is therefore a
necessary foundation, but it is not the paper's endpoint.  Our main question
is when random arrivals of contextual evidence induce a well-defined
stochastic process on semantic probability simplices.  This requires a bridge
that existing lexical calculi and generic stochastic systems do not generally
supply: a measurable quotient must preserve enough future behavior for random
transformations to descend, compose and remain stable through time.  The
existence, causal uniqueness, synchronization and perturbation results for
that descended process turn the lexical foundation into a stochastic lexical
calculus and place language-derived uncertainty within a compositional theory
of state-based computation.

The closest modern lines of work are coalgebraic
behavioral metrics and fibrational proof techniques
\cite{baldan2018,bonchi2023}, restriction categories for partial computation
\cite{cockettlack2002,cockett2025}, categorical probability and Markov
categories \cite{fritz2022,jacobs2023}, and renewed categorical treatments of
natural-language semantics \cite{chatzikyriakidis2026,cousin2023}.  These
theories explain behavior, partiality, stochastic composition and linguistic
structure at a high level of generality.  They do not by themselves specify
which transformations of a contextual record preserve statistical
information, how a probability law over verbal continuations is coarsened to
an application state, or how failure of that coarsening is exposed by a
commutation defect.  The present paper occupies that interface rather than
claiming an alternative foundation to those literatures.

The contribution claimed here is therefore more specific. We formulate a
stochastic lexical calculus in which:
\begin{enumerate}
\item the primitive domain consists of contextual lexical expressions;
\item typed transformations are the directions of change;
\item equivalence is operational and relative to declared observables and interventions;
\item the canonical state is constructed from future transform--observation behavior;
\item a proposed semantic representation is valid only when transformations descend through it;
\item random descended transformations define a simplex-valued recursion with
explicit existence and stability conditions; and
\item violations of the calculus laws are measurable rather than assumed away.
\end{enumerate}

To the best of our literature search, this combination---particularly terminal
minimality for typed partial measurable language transformations and testable
semantic descent---has not been developed for contextual natural language.
The constituent quotient, factorization, random-iteration and finite-difference
arguments are standard.  Our contribution is their lexical construction and
the resulting separation of three questions: what information a representation
retains, whether admitted transformations descend through it, and whether the
descended recursion is stable.

The theoretical development is organized around three principal results.
First, the behavioral
signature gives the smallest closed observable representation in the declared
category.  Second, exact descent is equivalent to saturation of transformation
domains and preservation of representation fibres; approximate descent has an
irreducible fibre-diameter error.  Third, once descent has been established,
standard average-contraction arguments yield a unique causal stochastic
recursion and an explicit propagation bound for descent error.  Statistical
identification and finite-sample recovery are included only to make the
observable implications testable; they are not presented as a new general
theory of semiparametric estimation.  Nor do the results imply that a fitted
language model literally contains a Bayesian belief state.

The exposition follows the same order as the certification problem.
Section~2 defines contextual transformations, information equivalence and the
canonical behavioral representation.  Section~3 develops semantic descent,
error propagation and finite examples.  Section~4 constructs the stochastic
lexical calculus on probability simplices.  Section~5 reports the completed
bounded validation, including the gates that failed, and Section~6 concludes.
The appendices contain the categorical foundations, statistical recovery
results and extended comparison with related mathematical structures.

\section{Contextual transformation systems and information equivalence}

\subsection{Contextual lexical spaces and transformations}

Before asking whether a semantic state is stable, we must say what is being
changed and what can be observed.  This is less trivial for language than for
a vector in Euclidean space.  Replacing a number in a structured record,
reordering two sentences, and inserting a negation are all string operations,
but they have different scientific roles.  The first may change evidence, the
second may preserve it, and the third may reverse meaning.  We therefore make
the transformation type and its legal domain part of the mathematical object.

For example, consider the record ``temperature is 39 C; oxygen saturation is
falling.''  Replacing 39 by 37 changes the measured evidence.  Reversing the
order of the two clauses can preserve the same evidence.  Inserting
``not'' before ``falling'' changes the second observation to its negation.
Although each operation edits a string, only the reordering is naturally
treated as information preserving for this example.

With these distinctions in place, we now formalize the contextual language
space, its observable outputs, and the typed transformations that act upon it.
Let $\Sigma$ be a finite or countable vocabulary and let $\Sigma^*$ denote the free monoid of finite strings under concatenation. The admissible context space $\Ctxt\subseteq\Sigma^*$ is equipped with a sigma-algebra $\mathcal A$. The continuation space $\Words\subseteq\Sigma^*$ has sigma-algebra $\mathcal G$.

An isolated word type is generally not a sufficient unit: the occurrence of \emph{strong} in ``strong earnings'' differs from its occurrence in ``strong inflation.'' We therefore take a lexical occurrence to be a triple $(c,i,w_i)$ consisting of a context, a location, and the expression occupying that location. Phrase spans and structured evidence records are included by allowing $i$ to index a finite interval or a typed field.

\begin{definition}[Contextual lexical system]
A contextual lexical system is a tuple
\[
 \mathfrak L=(\Sigma,\Ctxt,\Trans,\mathcal O,\simeq),
\]
where $\Trans$ is a category of typed partial transformations of contexts, $\mathcal O$ is a family of measurable lexical observables, and $\simeq$ is a declared information or semantic equivalence relation. An observable is a map $F:\Ctxt\to V_F$ into a specified measurable, metric, algebraic, or normed codomain.
\end{definition}

Examples of observables include the presence of a phrase, the provenance record recovered from a passage, a formal denotation, a distribution over possible continuations, a semantic orientation, or a probability composition over declared states. The theory does not require every observable to be probabilistic.

\begin{definition}[Primitive lexical edits]
For a contextual occurrence or span, primitive edits include typed substitution $R_{u\to v}$, insertion $I_v$, deletion $D_u$, permutation $P_\pi$, qualification $Q_v$, and negation $N$. Their legal domains are part of their definitions. Composite transformations are finite well-typed paths of primitive edits.
\end{definition}

The category formulation prevents arbitrary string manipulation from being treated as meaningful calculus. For example, a provenance-preserving reordering and a chronology-reversing reordering have different types even when both are permutations of the same words.

\subsection{Probability-valued lexical observables}

The following important representation is a special case, not the definition of the calculus.

\begin{definition}[Observable lexical kernel]
An observable lexical kernel is a Markov kernel
\[
 \Kern:\Ctxt\times\mathcal G\longrightarrow[0,1],
\]
such that $A\mapsto\Kern(c,A)$ is a probability measure for every $c\in\Ctxt$ and $c\mapsto\Kern(c,A)$ is measurable for every $A\in\mathcal G$. We write $\Kern_c:=\Kern(c,\cdot)\in\Prob(\Words)$.
\end{definition}

The kernel is the externally observable conditional distribution over verbal continuations. An observation mechanism may expose only a coarsening or truncation of this law; that case is represented by a separate observation operator and must not be silently identified with $\Kern_c$.

\begin{definition}[Semantic map]
For $K\ge2$, a semantic map is a measurable function
\[
 \Pi:\Prob(\Words)\longrightarrow\Delta^{K-1},
 \qquad
 \Delta^{K-1}:=\{\bm p\in[0,1]^K:\bm 1^\top\bm p=1\}.
\]
The observable semantic state is
\[
 \Bmap:=\Pi\circ\Kern:\Ctxt\longrightarrow\Delta^{K-1}.
\]
\end{definition}

For a measurable partition $A_1,\ldots,A_K$ of $\Words$, the simplest pushforward is
\[
\Pi(\mu)=\bigl(\mu(A_1),\ldots,\mu(A_K)\bigr).
\]
Calibrated maps may be nonlinear; none of the descent results below requires linearity.

\begin{definition}[Typed lexical transformation]
A lexical transformation is a measurable partial map $T:\operatorname{dom}(T)\subseteq\Ctxt\to\Ctxt$ supplied with a declared type, such as insertion, deletion, substitution, negation, permutation, paraphrase, or evidence update. A collection $\Trans$ is composition-closed whenever domains are compatible and contains the identity transformation.
\end{definition}

The type is part of the experimental intervention. The theory does not infer from the model alone whether a transformation is a paraphrase or a contradiction.

\subsection{A concrete foundation for information equivalence}

The word ``equivalent'' is dangerous unless the experiment relative to which
it is asserted has been fixed.  For example, ``temperature 39 C; oxygen
falling'' and ``oxygen falling; temperature 39 C'' may encode the same
clinical record, although a fitted language system can assign different
continuation probabilities to the two strings.  By contrast, replacing
``falling'' with ``stable'' changes the record.  Similarity of embeddings or
agreement of modal responses cannot distinguish these cases by itself.

We therefore begin with Blackwell's comparison of statistical experiments
\cite{blackwell1953} and then move from experiment-level equivalence to a
realized-record criterion that can be constructed in data.  This places prompt
perturbations within a familiar decision-theoretic framework: two
presentations count as information equivalent only when one contains no
state-relevant information absent from the other.

Information equivalence cannot mean that two passages have similar embeddings, share vocabulary, or receive the same label from a language model. It must be defined relative to a statistical question. Let $(\mathcal X,\mathcal H)$ be a measurable state space and let an information experiment be a Markov kernel
\[
 \Exp:\mathcal X\times\mathcal Y\longrightarrow[0,1],
\]
where $\Exp(x,\cdot)$ is the distribution of an observable record $Y$ under state $x$. A rendering map $r:\mathcal Y\to\Ctxt$ converts a structured record into language while retaining its source identifiers and values.

\begin{definition}[Blackwell information equivalence \cite{blackwell1953}]\label{def:blackwell}
Two experiments $\Exp_1:\mathcal X\rightsquigarrow\mathcal Y_1$ and $\Exp_2:\mathcal X\rightsquigarrow\mathcal Y_2$ are information equivalent, written $\Exp_1\simeq_{\mathrm B}\Exp_2$, if there exist state-independent Markov kernels $Q_{12}$ and $Q_{21}$ such that
\[
 \Exp_2=\Exp_1Q_{12},\qquad \Exp_1=\Exp_2Q_{21}.
\]
\end{definition}

This is the standard equivalence induced by Blackwell's comparison of statistical experiments: each observation can be simulated from the other without access to the state. Consequently the two experiments have the same attainable Bayes risk for every bounded decision problem.

For a fixed realized record, a more operational criterion is available. Suppose $\mathcal X=\{1,\ldots,K\}$ and the experiment is dominated by a measure $\nu$, with likelihood $e_x(y)$. Define the likelihood ray
\[
 [\bm e(y)]:=\{a(e_1(y),\ldots,e_K(y)):a>0\}.
\]

\begin{definition}[Realized information equivalence]\label{def:ray}
Two records $y,y'$ are likelihood equivalent, written $y\simeq_{\mathrm L}y'$, when
\[
 e_x(y')=a(y,y')e_x(y)\quad\text{for every }x\in\mathcal X
\]
for some scalar $a(y,y')>0$ independent of $x$.
\end{definition}

\begin{proposition}[Posterior characterization; cf. Blackwell equivalence \cite{blackwell1953}]\label{prop:posterior-equivalence}
For strictly positive likelihood vectors, $y\simeq_{\mathrm L}y'$ if and only if the Bayesian posterior after $y$ equals the posterior after $y'$ for every full-support prior on $\mathcal X$.
\end{proposition}
\begin{proof}
Proportional likelihoods cancel to the same normalized posterior. Conversely, equality under every full-support prior implies equality of every pairwise posterior odds; hence $e_x(y)/e_j(y)=e_x(y')/e_j(y')$ for all $x,j$, which is equivalent to proportionality.
\end{proof}

\begin{definition}[Provenance-preserving rendering equivalence]
Let $y$ be a structured evidence record carrying values, units, times, sources, and target definitions. Two contexts $c=r(y)$ and $c'=r'(y)$ are provenance-preserving renderings of the same record when both renderings are injective on the declared information fields and a deterministic audit map recovers $y$ from either context. We write $c\simeq_{\mathrm P}c'$.
\end{definition}

Provenance-preserving equivalence is deliberately stronger than ordinary paraphrase. It gives an experimentally constructible subset of Blackwell-equivalent presentations: both strings deterministically encode the same statistical record. Evidence order, bullet versus prose format, and equivalent numerical notation can be tested this way without claiming that unrestricted natural-language paraphrases are automatically equivalent.

Let $\simeq_{\mathcal I}$ denote the equivalence relation selected for an experiment, where $\mathcal I$ records the target state, decision class, horizon, and provenance rules. The quotient
\[
 \Ctxt_{\mathcal I}:=\Ctxt/{\simeq_{\mathcal I}}
\]
is the information-context space. The quotient is therefore task relative, not a universal partition of language.

\begin{definition}[Information-preserving transformation]
A transformation $T$ is information preserving on $\Ctxt_0$ when
\[
 c\simeq_{\mathcal I}Tc
\quad\text{for every }c\in\Ctxt_0.
\]
It is provenance preserving when the equivalence can be certified by recovery of the same structured record.
\end{definition}

This separates two empirical questions. Whether $T$ preserves information is certified from the data-generating experiment and provenance; whether the language service preserves the corresponding semantic state is tested through $\Bmap(Tc)\approx\Bmap(c)$.

\subsubsection{Lexical observational equivalence}

Blackwell and likelihood equivalence require a specified statistical experiment. A lexical calculus also needs an internal analogue that applies to nonprobabilistic observables.

\begin{definition}[Observational lexical equivalence]\label{def:lexequiv}
For observables $\mathcal O_0\subseteq\mathcal O$ and transformations $\Trans_0\subseteq\Trans$, define
\[
 c\simeq_{\mathcal O_0,\Trans_0}c'
 \quad\Longleftrightarrow\quad
 F(Tc)=F(Tc')
\]
for every $F\in\mathcal O_0$ and every $T\in\Trans_0$ for which both sides are defined.
\end{definition}

Two contexts are lexically equivalent when no admitted observable, either now or after an admitted intervention, distinguishes them. This is relative to a declared experimental vocabulary. If $\mathcal O_0$ contains all bounded decision losses generated by an experiment, it recovers decision-theoretic equivalence. If it contains all posterior-odds observables, it recovers likelihood-ray equivalence. If it contains an auditable record-recovery map, it respects provenance equivalence.

\begin{proposition}[Compatibility hierarchy]\label{prop:hierarchy}
Let $c=r(y)$ and $c'=r'(y')$ be renderings of observations from a dominated finite-state experiment.
\begin{enumerate}
\item Provenance equivalence of renderings of the same record implies likelihood equivalence.
\item Likelihood equivalence implies equality of every Bayesian posterior observable for every full-support prior.
\item Blackwell-equivalent experiments induce the same attainable risk for every bounded decision problem.
\item None of the converses holds without additional injectivity or richness assumptions.
\end{enumerate}
\end{proposition}
\begin{proof}
The first statement follows because the recovered record and hence its likelihood vector are identical. The second is Proposition~\ref{prop:posterior-equivalence}. The third is Blackwell's comparison theorem. Coarsened renderings and restricted decision or observable classes give counterexamples to the converses.
\end{proof}

Observational lexical equivalence therefore does not compete with Blackwell equivalence. It extends the same operational principle to a chosen family of contextual transformations and lexical observables.

\subsection{The canonical lexical representation}

The foundational problem is to construct the least informative representation that nevertheless supports all selected lexical observations after all selected transformations.

This is best understood as a memory question.  Suppose two contexts produce
the same response today.  They should be merged only if no admissible sequence
of later interventions can make the selected observations distinguish them.
A state that remembers less is not closed; a state that remembers more carries
unnecessary lexical detail.  This is the future-behavior principle behind
Myhill--Nerode minimization and coalgebraic behavior
\cite{nerode1958,rutten2000}.  We formulate that principle for typed, partial
and measurable contextual transformations with heterogeneous readouts.

Assume that $\Trans_0$ contains the identity and is closed under right composition by every transformation whose dynamics are modelled. Define
\[
 \mathcal Z_*:=\prod_{(F,T)\in\mathcal O_0\times\Trans_0}V_F
\]
with its product sigma-algebra.

\begin{definition}[Canonical lexical signature]\label{def:signature}
The canonical lexical signature is
\[
 \mathsf S_*:\Ctxt\longrightarrow\mathcal Z_*,
 \qquad
 \mathsf S_*(c):=\bigl(F(Tc)\bigr)_{(F,T)\in\mathcal O_0\times\Trans_0}.
\]
\end{definition}

Its fibres are exactly the equivalence classes in Definition~\ref{def:lexequiv}.

\begin{theorem}[Canonical minimal closed representation; extension of future-behavior quotients \cite{nerode1958,rutten2000}]\label{thm:canonical}
The signature $\mathsf S_*$ has the following properties.
\begin{enumerate}
\item For each modelled $U$, there is a unique induced transformation $\overline U_*$ on $\mathsf S_*(\Ctxt)$ satisfying
\[
 \mathsf S_*\circ U=\overline U_*\circ\mathsf S_*.
\]
\item Every $F\in\mathcal O_0$ factors through $\mathsf S_*$.
\item Let $R:\Ctxt\to\mathcal Z$ be any representation such that, for every $F\in\mathcal O_0$ and $T\in\Trans_0$, there is a map $g_{F,T}$ with
\[
 F\circ T=g_{F,T}\circ R.
\]
Then $\mathsf S_*$ factors uniquely through $R$ on $R(\Ctxt)$:
\[
 \mathsf S_*=h\circ R.
\]
Consequently, $R$ cannot identify two contexts distinguished by $\mathsf S_*$. Up to a bijection of its image, $\mathsf S_*$ is the coarsest representation sufficient for the declared observables and transformations.
\end{enumerate}
\end{theorem}

\begin{proof}
If $\mathsf S_*(c)=\mathsf S_*(c')$, then $F(Tc)=F(Tc')$ for every $(F,T)$. Closure under right composition gives $F(TUc)=F(TUc')$, hence $\mathsf S_*(Uc)=\mathsf S_*(Uc')$. The descent theorem therefore defines a unique $\overline U_*$. Since the identity lies in $\Trans_0$, $F(c)$ is a coordinate projection of $\mathsf S_*(c)$, proving the second claim. For the third, define
\[
 h(R(c)):=\bigl(g_{F,T}(R(c))\bigr)_{F,T}.
\]
The assumed factorizations make this definition representative independent and equal to $\mathsf S_*(c)$. Uniqueness follows on $R(\Ctxt)$.
\end{proof}

\begin{remark}[Why the theorem matters]
The theorem constructs rather than assumes the appropriate lexical state. A proposed semantic map is adequate only if it retains the distinctions made by $\mathsf S_*$, approximately or exactly. It also identifies what must be added when closure fails: an observable or transformation coordinate separating the offending fibre. This supplies a principled bridge from lexical representation to subsequent dynamics.
\end{remark}

\long\gdef\appendixcategorical{
\section{Typed partial measurable systems and coalgebraic recovery}

The one-space construction hides two features that matter in applications.
Not every operation is legal everywhere, and operations may move between
different sorts, such as a structured record, a rendered prompt and a
probability-valued readout.  Category and restriction-category language
provides bookkeeping for these domains and sorts.  It prevents ill-typed
compositions rather than adding abstraction for its own sake.

The construction is not intrinsically linguistic. We first define it for a general observable transformation system and later specialize contexts to natural language. Let $I$ be a countable set of sorts and let $\mathbf C=((C_i,\mathcal A_i))_{i\in I}$ be measurable spaces. For $i,j\in I$, a partial measurable arrow $T:i\dashrightarrow j$ is an equivalence class of pairs $(D_T,t)$, where $D_T\in\mathcal A_i$ and $t:(D_T,\mathcal A_i|_{D_T})\to(C_j,\mathcal A_j)$ is measurable; pairs are identified when they have the same domain and agree pointwise. For $T:i\dashrightarrow j$ and $U:j\dashrightarrow k$, set
\[
 D_{U\circ T}:=D_T\cap t^{-1}(D_U),
 \qquad (U\circ T)(c):=u(t(c)).
\]
The identity $1_i:i\dashrightarrow i$ is $(C_i,\operatorname{id}_{C_i})$. Associativity follows from associativity of ordinary composition and equality of the displayed domains. These data form the many-sorted partial-map category $\mathbf{ParMeas}(\mathbf C)$. An \emph{admissible transformation category} $\mathcal T$ is a small wide subcategory of $\mathbf{ParMeas}(\mathbf C)$.

For every sort $j$, let $\mathcal O_j$ be a countable family of measurable maps $F:C_j\to V_F$, where each $V_F$ is standard Borel. For $T:i\dashrightarrow j$, introduce the pointed space $V_F^\partial:=V_F\sqcup\{\partial\}$ and define
\[
 [F,T](c):=
 \begin{cases}
  F(Tc),&c\in D_T,\\
  \partial,&c\notin D_T.
 \end{cases}
\]
The cemetery value records inadmissibility and therefore distinguishes an undefined operation from a defined operation whose observable value happens to be zero.

\begin{definition}[Behavioral signature and quotient sigma-algebra]\label{def:typed-signature}
For sort $i$, let
\[
 P_i:=\prod_{j\in I}\prod_{T\in\mathcal T(i,j)}
       \prod_{F\in\mathcal O_j}V_F^\partial
\]
with its product sigma-algebra, and define
\[
 S_i:C_i\longrightarrow P_i,
 \qquad S_i(c):=\bigl([F,T](c)\bigr)_{j,T,F}.
\]
Write $Z_i^*:=S_i(C_i)$ and equip it with the final sigma-algebra
\[
 \mathcal Z_i^*:=\{A\subseteq Z_i^*:S_i^{-1}(A)\in\mathcal A_i\}.
\]
Thus $S_i:C_i\twoheadrightarrow Z_i^*$ is a measurable quotient map by construction.
\end{definition}

\begin{definition}[Closed observable quotient representation]\label{def:closed-rep}
A closed representation of $(\mathbf C,\mathcal T,\mathcal O)$ is a family of measurable quotient maps $R_i:C_i\twoheadrightarrow Z_i$, where $Z_i$ carries the final sigma-algebra induced by $R_i$, together with:
\begin{enumerate}
\item for every $T:i\dashrightarrow j$, a partial measurable map $\overline T:Z_i\dashrightarrow Z_j$ satisfying
\[
D_T=R_i^{-1}(D_{\overline T}),\qquad
R_j\circ T=\overline T\circ R_i\quad\text{on }D_T;
\]
\item for every $F\in\mathcal O_i$, a measurable readout $f_F:Z_i\to V_F$ satisfying $F=f_F\circ R_i$.
\end{enumerate}
A morphism $h:R\to R'$ is a family of measurable maps $h_i:Z_i\to Z_i'$ such that $R_i'=h_i\circ R_i$. This equality forces preservation of readouts, domains and induced transitions. Closed representations and these morphisms form a category $\mathbf{Rep}(\mathbf C,\mathcal T,\mathcal O)$.
\end{definition}

The orientation of morphisms is deliberate: an arrow $R\to R'$ means that $R'$ is a measurable coarsening of $R$.

\begin{theorem}[Terminal behavioral representation---typed measurable extension]\label{thm:typed-minimal}
Assume $I$, the hom-sets of $\mathcal T$, and the observable families are countable. Then $S=(S_i)_{i\in I}$ is a closed representation and is terminal in $\mathbf{Rep}(\mathbf C,\mathcal T,\mathcal O)$. Consequently:
\begin{enumerate}
\item every $T:i\dashrightarrow j$ has a unique induced partial measurable map $\overline T_*:Z_i^*\dashrightarrow Z_j^*$;
\item every selected observable factors uniquely through $S$;
\item every closed representation $R$ admits a unique measurable morphism $R\to S$;
\item the terminal representation is unique up to unique measurable isomorphism.
\end{enumerate}
Hence $S$ is the coarsest measurable quotient that preserves admissibility, selected observations and all future behavior under $\mathcal T$.
\end{theorem}

\begin{proof}
Countability and componentwise measurability make every $S_i$ measurable. Suppose $S_i(c)=S_i(c')$. The cemetery coordinates imply $c\in D_T$ if and only if $c'\in D_T$. If $T:i\dashrightarrow j$ is admissible at both points, then for every $U:j\dashrightarrow k$ and $F\in\mathcal O_k$,
\[
 [F,U](Tc)=[F,U\circ T](c)=[F,U\circ T](c')=[F,U](Tc').
\]
Thus $S_j(Tc)=S_j(Tc')$, and
\[
 \overline T_*(S_i(c)):=S_j(Tc),
 \qquad D_{\overline T_*}:=S_i(D_T),
\]
is representative independent. Domain saturation follows from the same cemetery coordinate. To prove measurability, let $A\in\mathcal Z_j^*$. By the definition of the final sigma-algebra,
\[
 S_i^{-1}\!\left(\overline T_*^{-1}(A)\right)
 =D_T\cap T^{-1}(S_j^{-1}(A))\in\mathcal A_i;
\]
hence $\overline T_*^{-1}(A)\in\mathcal Z_i^*$. The identity coordinate gives the readout factorization, so $S$ is closed.

Let $R$ be any closed representation. Define $h_i:Z_i\to Z_i^*$ by $h_i(R_i(c))=S_i(c)$. If $R_i(c)=R_i(c')$, domain preservation, repeated transition closure and readout factorization imply $[F,T](c)=[F,T](c')$ for every coordinate; hence $h_i$ is well defined. Since $S_i=h_i\circ R_i$ and $Z_i$ has the final sigma-algebra of $R_i$, $h_i$ is measurable. Surjectivity of $R_i$ gives uniqueness. Thus $S$ is terminal. The final assertion is the usual uniqueness of terminal objects.
\end{proof}

\begin{remark}
Classical deterministic automata arise when $I$ is a singleton, every arrow is total, $\mathcal T$ is a free-monoid action and there is one finite-valued readout. The additional structure here is partial admissibility, multiple sorts, heterogeneous measurable readouts and a quotient sigma-algebra that supports statistical observation.
\end{remark}

\subsection{Restriction-category structure}

Categories of partial maps carry a canonical restriction operation \cite{cockettlack2002}. For $T=(D_T,t):i\dashrightarrow j$, define
\[
 \overline T:=(D_T,\operatorname{id}_{D_T}):i\dashrightarrow i.
\]
This partial identity records exactly where $T$ is defined.

\begin{proposition}[Restriction structure \cite{cockettlack2002}]\label{prop:restriction}
The category $\mathbf{ParMeas}(\mathbf C)$ is a restriction category. With ordinary right-to-left composition, its restriction operation satisfies
\begin{align*}
T\circ\overline T&=T,\\
\overline T\circ\overline U&=\overline U\circ\overline T
   &&\text{when $T,U$ have the same source},\\
\overline{T\circ\overline U}&=\overline T\circ\overline U
   &&\text{when $T,U$ have the same source},\\
\overline U\circ T&=T\circ\overline{U\circ T}
   &&\text{when $U\circ T$ is typed}.
\end{align*}
Every admissible transformation category $\mathcal T$ that contains the restrictions of its arrows is therefore a restriction subcategory. Moreover, the induced partial maps in every closed representation preserve restrictions:
\[
 \overline{\,\overline T_*\,}=\overline{(\overline T)_*}.
\]
\end{proposition}

\begin{proof}
The first identity holds because restricting to $D_T$ before applying $T$ changes neither its domain nor its values. The second follows because both composites are the partial identity on $D_T\cap D_U$. The third has domain $D_T\cap D_U$ and is the identity there. For the fourth, both sides have domain
\[
D_T\cap T^{-1}(D_U)
\]
and agree with $T$ on that domain. These are the restriction axioms. A closed representation satisfies $D_T=R_i^{-1}(D_{\overline T_*})$; hence the restriction of the induced arrow is the partial identity on the image of the saturated domain, which is precisely the arrow induced by $\overline T$.
\end{proof}

The restriction formulation is not an alternative theory. It identifies the established categorical structure underlying admissibility and clarifies the new step: constructing a terminal observable quotient inside a measurable, many-sorted restriction system.

\subsection{Recovery of the classical final-coalgebra behavior map}

Consider the total, one-sorted case. Let $A$ be a countable action alphabet, let $(O,\mathcal O)$ be a measurable output space, and let
\[
H(X):=O\times X^A
\]
on measurable spaces and measurable maps, with countable-product sigma-algebras. A measurable Moore system is an $H$-coalgebra
\[
\gamma:C\longrightarrow O\times C^A,
\qquad
\gamma(c)=\bigl(o(c),(T_a c)_{a\in A}\bigr).
\]
For $w=a_1\cdots a_n\in A^*$, write $T_w=T_{a_n}\circ\cdots\circ T_{a_1}$ and $T_\epsilon=1_C$.

\begin{theorem}[Coalgebraic recovery; standard Moore behavior theorem \cite{loregian2025,rutten2000}]\label{thm:coalgebra-recovery}
The measurable space $O^{A^*}$, equipped with
\[
\zeta(q)=\bigl(q(\epsilon),(q_a)_{a\in A}\bigr),
\qquad q_a(w):=q(aw),
\]
is a final $H$-coalgebra. Its unique coalgebra morphism from $(C,\gamma)$ is
\[
\beta:C\longrightarrow O^{A^*},
\qquad
\beta(c)(w)=o(T_w c).
\]
If the selected observable family is $\{o\}$ and the transformation category is the action of $A^*$, then the lexical behavioral signature $S$ is exactly $\beta$. Consequently its terminal quotient agrees with the ordinary Moore behavioral quotient and, for Boolean language acceptance, with the Myhill--Nerode quotient.
\end{theorem}

\begin{proof}
Countability of $A^*$ makes $O^{A^*}$ measurable with the product sigma-algebra, and every coordinate of $\zeta$ is a coordinate projection, so $\zeta$ is measurable. The displayed $\beta$ is measurable coordinatewise. It is a coalgebra morphism because
\[
\beta(c)(\epsilon)=o(c),
\qquad
\beta(T_a c)(w)=o(T_wT_a c)=\beta(c)(aw).
\]
If $f:C\to O^{A^*}$ is any coalgebra morphism, its output equation gives $f(c)(\epsilon)=o(c)$ and its transition equation gives $f(c)(aw)=f(T_a c)(w)$. Induction on word length yields $f(c)(w)=o(T_w c)$, so $f=\beta$. Finally, the coordinates of $S(c)$ are precisely $(o(T_w c))_{w\in A^*}$; hence $S=\beta$, and equality of signatures is classical future-output equivalence.
\end{proof}

\begin{proposition}[Strictness beyond total Moore systems]\label{prop:strict-extension}
The framework contains countable-action measurable Moore systems behaviorally faithfully, by Theorem~\ref{thm:coalgebra-recovery}, but it is strictly more expressive when transformation admissibility is observable. In particular, let $C=\{c_0,c_1\}$, let the only readout be constant, and let $T$ be defined only at $c_0$. The behavioral signature separates $c_0$ and $c_1$ through the cemetery coordinate. No total Moore system on the same carrier, with the same constant readout and action label, can represent this distinction. It can do so only by changing the model---for example by adjoining a failure state or an admissibility readout.
\end{proposition}

\begin{proof}
The inclusion of total Moore systems and preservation of their behavioral equivalence follow from Theorem~\ref{thm:coalgebra-recovery}. In the partial system, $[F,T](c_0)=F(Tc_0)$ while $[F,T](c_1)=\partial$, so $S(c_0)\ne S(c_1)$. In a total Moore system on the same two points, $T$ must be defined at both. Since the present readout is constant and no admissibility coordinate exists, the undefined-versus-defined distinction cannot be expressed. Adding a failure state or readout changes the carrier or observable signature, proving strictness relative to the stated total theory.
\end{proof}

}

\subsection{Axioms of lexical calculus}

The representation theory above determines when a closed state exists.  We
next record the elementary laws obeyed by finite changes of an observable.
These laws act as a ledger: they track change under one edit, accumulation
along a sequence, and order effects between noncommuting edits.  They do not
postulate a smooth manifold of sentences or an infinitesimal derivative of a
word.

For example, if $T$ appends an evidence item and $S$ changes its numerical
format, the total change from $c$ to $S(Tc)$ is the change caused by $T$ plus
the change caused by $S$ at the already transformed context $Tc$.  This is an
exact cocycle identity, not a Taylor approximation.

We now state the laws required of the deterministic calculus. They define the target mathematical structure. An observed language service need not satisfy them exactly; its departures are reported as defects.

Let $\mathsf{Obs}(\Ctxt,V)$ be a vector space of measurable observables $F:\Ctxt\to V$, where $V$ is a normed vector space. Let $\Trans$ be a small category whose objects are admissible context domains and whose arrows are typed lexical transformations. Composition is written $S\circ T$, meaning first $T$, then $S$.

\begin{description}
\item[LC1: Typed closure.] The identity belongs to $\Trans$, and compatible transformations have an associative composition. Transformations that violate the declared grammar, provenance, or domain constraints are not composable arrows.

\item[LC2: Linearity in observables.] For scalars $a,b$,
\[
 \Dlex_T(aF+bG)=a\Dlex_TF+b\Dlex_TG.
\]

\item[LC3: Null and identity laws.]
\[
 \Dlex_{\Id}F=0,
 \qquad
 \Dlex_TF=0\ \text{whenever }F\circ T=F.
\]

\item[LC4: Cocycle law.]
\[
 \Dlex_{S\circ T}F(c)=\Dlex_TF(c)+\Dlex_SF(Tc).
\]

\item[LC5: Product law.] For scalar observables $F,G$,
\[
 \Dlex_T(FG)(c)
 =F(c)\Dlex_TG(c)+G(c)\Dlex_TF(c)
  +\Dlex_TF(c)\Dlex_TG(c).
\]
Equivalently, $\Dlex_T(FG)=F\Dlex_TG+(G\circ T)\Dlex_TF$.

\item[LC6: Information congruence.] If $c\simeq_{\mathcal I}c'$, then an information-invariant observable satisfies $F(c)=F(c')$. Moreover, every admissible transformation declared to act on information classes obeys
\[
 c\simeq_{\mathcal I}c'\Longrightarrow Tc\simeq_{\mathcal I}Tc'.
\]
Thus it descends to $\Ctxt_{\mathcal I}$.

\item[LC7: Semantic naturality.] A dynamically adequate semantic map satisfies
\[
 \Bmap\circ T=\overline T\circ\Bmap
\]
for every transformation in the modelled family. Approximate calculi replace equality by a uniform metric defect.

\item[LC8: Pushforward compatibility.] For a measurable semantic map $\Pi$ and lexical kernel $\Kern$,
\[
 \Dlex_T(\Pi\circ\Kern)(c)
 =\Pi(\Kern_{Tc})-\Pi(\Kern_c).
\]
If $\Pi$ is bounded linear on signed measures, then
\[
 \Dlex_T(\Pi\circ\Kern)=\Pi(\Dlex_T\Kern).
\]

\item[LC9: Fundamental path law.] Along every finite admissible path, total change equals the sum of transported lexical differentials. Path independence holds exactly when circulation vanishes on all cycles.

\item[LC10: Continuity.] For specified metrics $d_{\Ctxt}$ and $d_V$, admissible observables and induced state maps have declared moduli of continuity. This converts approximate information equivalence and measurement error into explicit output-error bounds.
\end{description}

LC1--LC5 and LC9 are algebraic laws inherited from transformation and finite-difference calculus. LC6--LC8 are the language-and-semantics laws that distinguish the present construction. LC10 is required for statistical falsifiability.

\subsubsection{Empirical defects}

For each axiom, define a defect rather than treating approximate equality as success by inspection. Important examples are
\begin{align*}
 \delta_{\mathrm{info}}(T)
 &:=\sup_{c\in\Ctxt_0:\,c\simeq_{\mathcal I}Tc}
 d\bigl(\Bmap(c),\Bmap(Tc)\bigr),\\
 \delta_{\mathrm{nat}}(T,\overline T)
 &:=\sup_{c\in\Ctxt_0}
 d\bigl(\Bmap(Tc),\overline T(\Bmap(c))\bigr),\\
 \delta_{\mathrm{comp}}(S,T)
 &:=\sup_{\bm b}
 d\bigl(\overline{S\circ T}(\bm b),
          \overline S(\overline T(\bm b))\bigr).
\end{align*}
The calculus is $(\varepsilon_{\mathrm{info}},\varepsilon_{\mathrm{nat}},\varepsilon_{\mathrm{comp}})$-valid on an operating domain when simultaneous confidence bounds for these defects lie below thresholds frozen before testing.

\section{Minimal representations and semantic descent}

\subsection{The target commuting diagram}

The canonical signature may be much larger than the state an application
wants to retain.  A clinician may keep three triage probabilities; a control
system may keep normal, warning and critical probabilities.  The practical
question is therefore not whether some closed representation exists, but
whether this particular coarsening supports the proposed transformation.

The test is a commuting diagram.  Transforming the full context and then
measuring its state must agree with first measuring the state and then
applying an update defined solely on that state.  If two contexts have the
same retained state but the same transformation sends them to different
successor states, no state-only update exists.  This is a sufficiency failure,
not a finite-sample inconvenience.

For $T\in\Trans$, we seek an induced state transformation
\[
 \overline T:\Bmap(\operatorname{dom}T)\longrightarrow\Delta^{K-1}
\]
such that
\begin{equation}\label{eq:commute}
 \Bmap\circ T=\overline T\circ\Bmap.
\end{equation}

\begin{center}
\begin{tikzpicture}[node distance=30mm and 42mm,>=Latex]
\node (c) {$c\in\Ctxt$};
\node[right=of c] (tc) {$Tc\in\Ctxt$};
\node[below=of c] (b) {$\Bmap(c)\in\Delta^{K-1}$};
\node[below=of tc] (tb) {$\Bmap(Tc)\in\Delta^{K-1}$};
\draw[->] (c)--node[above]{$T$}(tc);
\draw[->] (c)--node[left]{$\Pi\circ\Kern$}(b);
\draw[->] (tc)--node[right]{$\Pi\circ\Kern$}(tb);
\draw[->] (b)--node[below]{$\overline T$}(tb);
\end{tikzpicture}
\end{center}

Equation \eqref{eq:commute} is semantic closure: the measured state retains exactly the information needed to determine the effect of $T$.

\subsubsection{Exact descent}

Define the semantic equivalence relation
\[
 c\sim_{\Bmap}c'\quad\Longleftrightarrow\quad\Bmap(c)=\Bmap(c').
\]

\begin{theorem}[Descent, existence, and uniqueness---quotient factorization]\label{thm:descent}
For a transformation $T$, the following statements are equivalent.
\begin{enumerate}
\item There exists $\overline T$ satisfying \eqref{eq:commute}.
\item $T$ preserves semantic fibres:
\[
 \Bmap(c)=\Bmap(c')\Longrightarrow\Bmap(Tc)=\Bmap(Tc')
\]
for every $c,c'\in\operatorname{dom}(T)$.
\item $T$ induces a well-defined map on the quotient $\operatorname{dom}(T)/{\sim_{\Bmap}}$.
\end{enumerate}
When these conditions hold, $\overline T$ is unique on $\Bmap(\operatorname{dom}T)$.
\end{theorem}

\begin{proof}
If (1) holds and $\Bmap(c)=\Bmap(c')$, then
$\Bmap(Tc)=\overline T(\Bmap(c))=\overline T(\Bmap(c'))=\Bmap(Tc')$, proving (2). Condition (2) makes the definition
$\overline T(\Bmap(c)):=\Bmap(Tc)$ independent of the representative $c$, proving (1) and (3). Any map satisfying \eqref{eq:commute} must take $\Bmap(c)$ to $\Bmap(Tc)$, which proves uniqueness on the image.
\end{proof}

\begin{remark}
Theorem~\ref{thm:descent} is a quotient-factorization result. Its importance here is diagnostic: failure is evidence that the proposed semantic state is not sufficient for the transformation family. The appropriate response is to enlarge or revise the ontology, not to fit increasingly elaborate dynamics to a nonclosed state.
\end{remark}

\subsubsection{Approximate descent}

Let $(\Delta^{K-1},d)$ be a metric space. Define the diameter of the transformed fibre at $\bm b$ by
\[
 \omega_T(\bm b):=\sup_{c,c':\,\Bmap(c)=\Bmap(c')=\bm b}
 d\bigl(\Bmap(Tc),\Bmap(Tc')\bigr),
\]
and $\omega_T^*:=\sup_{\bm b}\omega_T(\bm b)$.

\begin{theorem}[Sharp representation-relative approximate descent]\label{thm:approx}
Let $d$ be induced by a norm and assume that each transformed fibre has nonempty image. For any state map $G$,
\[
 \sup_c d\bigl(\Bmap(Tc),G(\Bmap(c))\bigr)\ge \frac{1}{2}\omega_T^*.
\]
If transformed fibre images admit Chebyshev centres of radius at most $r_T$, there exists $\overline T$ satisfying
\[
 \sup_c d\bigl(\Bmap(Tc),\overline T(\Bmap(c))\bigr)\le r_T,
\]
where $\omega_T^*/2\le r_T\le\omega_T^*$. In one-dimensional state coordinates, $r_T=\omega_T^*/2$.
\end{theorem}

\begin{proof}
For any two points in the same transformed fibre, the triangle inequality implies that at least one lies at distance at least half their separation from $G(\bm b)$. Taking suprema proves the lower bound. Selecting a Chebyshev centre in each transformed fibre proves the upper bound. On the real line the midpoint of the extrema has radius one-half the diameter.
\end{proof}

This theorem identifies the irreducible state-compression error. It cannot be repaired by collecting more observations while retaining the same semantic state.

For example, suppose two records are both summarized as
$(0.6,0.3,0.1)$, but the same new evidence sends one to
$(0.8,0.15,0.05)$ and the other to $(0.3,0.4,0.3)$.  The present three
numbers have omitted information needed for the update.
Theorem~\ref{thm:approx} turns that objection into a lower bound: every
state-only update must err on at least one record.

\subsubsection{Composition of approximate descents}

One-step closure is useful only if its error can be controlled under composition. Let $T_j:C_{j-1}\dashrightarrow C_j$ be a compatible path, let $B_j:C_j\to(M_j,d_j)$ be semantic representations, and let $G_j:M_{j-1}\to M_j$ be proposed induced updates. Define the exact represented path and its closed approximation by
\[
y_0=B_0(c),\qquad y_j=B_j(T_j\cdots T_1c),
\]
and
\[
\widehat y_0=B_0(c),\qquad \widehat y_j=G_j(\widehat y_{j-1}).
\]

\begin{theorem}[Multi-step approximate descent; discrete Gronwall bound]\label{thm:multistep}
Suppose $G_j$ is $L_j$-Lipschitz and, on every exact state reached by the path,
\[
d_j\bigl(B_j(T_jx),G_j(B_{j-1}(x))\bigr)\le\varepsilon_j.
\]
Then for every admissible initial context,
\[
d_n(y_n,\widehat y_n)
\le
\sum_{j=1}^{n}
\left(\prod_{k=j+1}^{n}L_k\right)\varepsilon_j,
\]
where an empty product equals one. In particular:
\begin{enumerate}
\item if $L_j\le1$ and $\varepsilon_j\le\varepsilon$, the error is at most $n\varepsilon$;
\item if $L_j\le\rho<1$ and $\varepsilon_j\le\varepsilon$, the error is at most $\varepsilon(1-\rho^n)/(1-\rho)$;
\item without a bound on the products of Lipschitz constants, uniformly small one-step defects need not yield a stable pathwise representation.
\end{enumerate}
\end{theorem}

\begin{proof}
Let $e_j=d_j(y_j,\widehat y_j)$. Inserting $G_j(y_{j-1})$ and applying the triangle inequality and Lipschitz property gives
\[
e_j
\le d_j(y_j,G_j(y_{j-1}))
   +d_j(G_j(y_{j-1}),G_j(\widehat y_{j-1}))
\le\varepsilon_j+L_je_{j-1}.
\]
Since $e_0=0$, iteration proves the displayed sum. The first two conclusions follow by bounding the products; the third follows because those products are the amplification factors in the exact bound.
\end{proof}

\long\gdef\appendixstatistical{
\section{Statistical recovery and identification}

\subsection{Uniform descent under estimated lexical kernels}

Let $(M,d)$ be a metric semantic space, let $B:\Ctxt\to M$ be a population representation, and let $\widehat B$ be an estimate. For a transformation $T$ and a candidate update $G:M\to M$, define
\begin{align*}
 R_T(G)
 &:=\sup_{c\in D_T}d\bigl(B(Tc),G(B(c))\bigr),\\
 \widehat R_T(G)
 &:=\sup_{c\in D_T}d\bigl(\widehat B(Tc),G(\widehat B(c))\bigr).
\end{align*}

\begin{theorem}[Uniform plug-in descent bound---semantic-map extension]\label{thm:uniform-descent}
Let $\mathcal G_L$ be any class of $L$-Lipschitz maps from $M$ to itself. If
\[
 \sup_{x\in\Ctxt_0\cup T(\Ctxt_0)}
 d\bigl(\widehat B(x),B(x)\bigr)\le\varepsilon_B,
\]
then
\[
 \sup_{G\in\mathcal G_L}
 \left|\widehat R_T(G)-R_T(G)\right|
 \le (1+L)\varepsilon_B.
\]
Suppose only a finite evaluation sample is available and an empirical criterion $\widehat R_{T,n}$ satisfies
\[
 \sup_{G\in\mathcal G_L}
 |\widehat R_{T,n}(G)-\widehat R_T(G)|\le\zeta_n.
\]
If $\widehat G$ is an $\eta_n$-minimizer of $\widehat R_{T,n}$, then
\[
 R_T(\widehat G)
 \le
 \inf_{G\in\mathcal G_L}R_T(G)
 +2(1+L)\varepsilon_B+2\zeta_n+\eta_n.
\]
Both conclusions hold simultaneously over a transformation family whenever the assumed bounds hold uniformly over that family.
\end{theorem}

\begin{proof}
For every $c$ and $G\in\mathcal G_L$, two applications of the triangle inequality give
\begin{align*}
 &\left|
 d(\widehat B(Tc),G(\widehat B(c)))
 -d(B(Tc),G(B(c)))
 \right|\\
 &\qquad\le
 d(\widehat B(Tc),B(Tc))
 +d(G(\widehat B(c)),G(B(c)))
 \le(1+L)\varepsilon_B.
\end{align*}
Taking suprema over contexts and then candidate maps proves the first claim. For the second, write $a=(1+L)\varepsilon_B+\zeta_n$ and let $G^*$ approach the population infimum. Then
\[
 R_T(\widehat G)
 \le\widehat R_{T,n}(\widehat G)+a
 \le\widehat R_{T,n}(G^*)+\eta_n+a
 \le R_T(G^*)+2a+\eta_n.
\]
Taking the infimum proves the result.
\end{proof}

\begin{corollary}[Observable-kernel error propagation]\label{cor:kernel-uniform}
Suppose $B=\Pi\circ\Kern$ and $\widehat B=\widehat\Pi\circ\widehat\Kern$. If
\[
 \sup_c d_{\Prob}(\widehat\Kern_c,\Kern_c)\le\varepsilon_K,
 \qquad
 \sup_{\mu}d(\widehat\Pi(\mu),\Pi(\mu))\le\varepsilon_\Pi,
\]
and $\Pi$ is $L_\Pi$-Lipschitz, then Theorem~\ref{thm:uniform-descent} holds with
\[
 \varepsilon_B=\varepsilon_\Pi+L_\Pi\varepsilon_K.
\]
If these inequalities hold with probability at least $1-\alpha$, the resulting descent bound has the same coverage.
\end{corollary}
\begin{proof}
Insert $\Pi(\widehat\Kern_c)$ between $\widehat\Pi(\widehat\Kern_c)$ and $\Pi(\Kern_c)$ and apply the triangle inequality and Lipschitz condition. The result then follows from Theorem~\ref{thm:uniform-descent}.
\end{proof}

The theorem separates three quantities: irreducible closure error $\inf_G R_T(G)$, representation-estimation error $\varepsilon_B$, and finite-evaluation error $\zeta_n$. Only the latter two vanish with additional data.

\subsection{An explicit finite-sample rate}

The supremum criterion above is appropriate for uniform closure certification. A standard expected-risk version gives an explicit rate from elementary concentration. Let $C_1,\ldots,C_n$ be independent draws from a distribution $P_T$ supported on $D_T$, define
\[
\mathcal R_T(G):=\mathbb E_{P_T}
d\bigl(B(TC),G(B(C))\bigr),
\]
and let $\widehat{\mathcal R}_{T,n}(G)$ be the corresponding sample average with $B$ replaced by $\widehat B$.

\begin{corollary}[Finite-class rate; Hoeffding bound \cite{hoeffding1963}]\label{cor:finite-rate}
Let $\mathcal G_L$ contain $N<\infty$ maps, each $L$-Lipschitz, and suppose the loss is bounded by $M$. If
\[
\sup_{c\in D_T\cup T(D_T)}d(\widehat B(c),B(c))\le\varepsilon_B,
\]
and $\widehat G$ minimizes $\widehat{\mathcal R}_{T,n}$, then, with probability at least $1-\alpha$,
\[
\mathcal R_T(\widehat G)
\le
\min_{G\in\mathcal G_L}\mathcal R_T(G)
+2(1+L)\varepsilon_B
+2M\sqrt{\frac{\log(2N/\alpha)}{2n}}.
\]
Consequently, for fixed $N$ and $\varepsilon_B=o(1)$, the excess expected closure risk is $O_{\mathbb P}(n^{-1/2})$.
\end{corollary}

\begin{proof}
For each fixed $G$, Hoeffding's inequality gives
\[
\Pr\left(
|\mathcal R_T(G)-\mathcal R_{T,n}(G)|>u
\right)
\le2\exp\left(-\frac{2nu^2}{M^2}\right).
\]
A union bound over $N$ maps and the choice
$u=M\sqrt{\log(2N/\alpha)/(2n)}$ give uniform deviation at most $u$ with probability $1-\alpha$. Replacing $B$ by $\widehat B$ perturbs every loss by at most $(1+L)\varepsilon_B$. Comparing the empirical minimizer with the population minimizer incurs each uniform error twice, yielding the result.
\end{proof}

The concentration inequality and union bound are standard; the contribution of Corollary~\ref{cor:finite-rate} is to expose how semantic-representation error and statistical update-selection error enter the closure certificate separately. Infinite function classes may replace $\log N$ by an appropriate covering-number or Rademacher-complexity term.

\subsection{Identification through enriched behavioral separation}

Closure and identification answer different questions.  Closure asks whether
a known representation supports a well-defined update.  Identification asks
whether observable interventions distinguish competing data-generating
systems or representations.  A perfectly closed constant state is not
informative; an identifiable but nonclosed score cannot be updated
autonomously.

Transformations therefore play the role of experimental interventions.  Two
systems that agree under one static prompt may respond differently to a
controlled negation, evidence insertion or information-equivalent reordering.
Adding such interventions can restore identification, provided the
intervention family is frozen before evaluation.

Let $\Theta$ index a family of lexical data-generating processes. For each $\theta\in\Theta$, transformation $T$, context design $c$, and observable $F$, let
\[
 \mathsf Q_\theta(F,T,c)
\]
denote the population law or value made observable by the experiment. Define the enriched behavioral signature
\[
 \mathsf Q(\theta)
 :=
 \bigl(\mathsf Q_\theta(F,T,c)\bigr)_{(F,T,c)\in\mathfrak I}
 \in\mathcal Q,
\]
where $\mathfrak I$ is the frozen intervention design and $(\mathcal Q,d_{\mathcal Q})$ is a metric product space. Let $\asymp$ be the scientific equivalence relation under which parameters representing the same lexical process are not distinguished.

\begin{definition}[Empirical lexical identifiability]
The model is identifiable modulo $\asymp$ when
\[
 \mathsf Q(\theta)=\mathsf Q(\theta')
 \Longrightarrow\theta\asymp\theta'.
\]
It is behaviorally separating when every pair $\theta\not\asymp\theta'$ is distinguished by at least one admitted triple $(F,T,c)$.
\end{definition}

\begin{theorem}[Identification--separation equivalence]\label{thm:id-separation}
For a fixed intervention design $\mathfrak I$, empirical lexical identifiability modulo $\asymp$ holds if and only if the observable family is behaviorally separating. Equivalently, the signature descends to an injective map
\[
 \widetilde{\mathsf Q}:\Theta/{\asymp}\longrightarrow\mathcal Q.
\]
If the quotient is compact, $\mathcal Q$ is Hausdorff, and $\widetilde{\mathsf Q}$ is continuous, then the inverse on its image is continuous. In particular, for
\[
 \kappa(r):=
 \inf\left\{
 d_{\mathcal Q}(\mathsf Q(\theta),\mathsf Q(\theta')):
 d_{\Theta/{\asymp}}([\theta],[\theta'])\ge r
 \right\},
\]
one has $\kappa(r)>0$ for every $r>0$.
\end{theorem}

\begin{proof}
Failure of behavioral separation means that some inequivalent pair agrees in every coordinate of $\mathsf Q$, which is precisely failure of identification. The converse is the contrapositive. Hence $\mathsf Q$ is constant on equivalence classes and injective after quotienting. A continuous injection from a compact space into a Hausdorff space is a homeomorphism onto its image. Positivity of $\kappa(r)$ follows because the set of pairs at quotient distance at least $r$ is compact and its continuous signature distance cannot attain zero under injectivity.
\end{proof}

\begin{theorem}[Finite-error recovery from behavioral signatures]\label{thm:id-rate}
Under the assumptions of Theorem~\ref{thm:id-separation}, suppose
\[
 d_{\mathcal Q}(\widehat{\mathsf Q},\mathsf Q(\theta_0))
 \le\varepsilon_n
\]
and $\widehat\theta$ is an $\eta_n$-minimum-distance estimator:
\[
 d_{\mathcal Q}(\widehat{\mathsf Q},\mathsf Q(\widehat\theta))
 \le
 \inf_{\theta\in\Theta}
 d_{\mathcal Q}(\widehat{\mathsf Q},\mathsf Q(\theta))
 +\eta_n.
\]
Then
\[
 d_{\mathcal Q}(\mathsf Q(\widehat\theta),\mathsf Q(\theta_0))
 \le2\varepsilon_n+\eta_n
\]
and
\[
 d_{\Theta/{\asymp}}([\widehat\theta],[\theta_0])
 \le
 \kappa^{-1}(2\varepsilon_n+\eta_n),
\]
where $\kappa^{-1}(u):=\sup\{r:\kappa(r)\le u\}$. Thus uniform signature consistency implies consistency of the identified lexical process; absence of behavioral separation makes such recovery impossible regardless of sample size.
\end{theorem}

\begin{proof}
The true parameter is feasible in the minimum-distance problem, so
\[
 d_{\mathcal Q}(\widehat{\mathsf Q},\mathsf Q(\widehat\theta))
 \le\varepsilon_n+\eta_n.
\]
The triangle inequality gives the first bound. If the quotient distance exceeded the stated inverse modulus, the definition of $\kappa$ would force the signature distance above $2\varepsilon_n+\eta_n$, a contradiction.
\end{proof}

Theorem~\ref{thm:id-separation} states the foundational equivalence; Theorem~\ref{thm:id-rate} supplies its statistical content. Transformations are not merely robustness checks: they enlarge the intervention design $\mathfrak I$ and can turn an unidentified static observable into an identified behavioral signature.

\subsection{Synthesis: representation, descent and observable recovery}

The preceding results answer different parts of one certification question and
are assembled here for reference.  The structural conclusions are
Parts~(I)--(III): they identify the canonical population representation and
determine whether a proposed lower-dimensional state is dynamically
legitimate.  Parts~(IV)--(V) state the additional observability conditions
needed to learn these objects from data.  Those parts use standard separation
and minimum-distance arguments so that failure of representation can be
distinguished from finite-sample error.

\begin{theorem}[Representation, descent and observable recovery]\label{thm:central}
Let $(\mathbf C,\mathcal T,\mathcal O)$ satisfy the countability assumptions of Theorem~\ref{thm:typed-minimal}, and let $S:\mathbf C\twoheadrightarrow\mathbf Z^*$ be its behavioral signature. Let $B_i:C_i\twoheadrightarrow M_i$ be measurable quotient maps into metric spaces, and let $\mathcal G_L(T)$ be a class of $L_T$-Lipschitz candidate updates $G:M_i\dashrightarrow M_j$ for every $T:i\dashrightarrow j$. Then the following conclusions hold.
\begin{enumerate}
\item[\textup{(I)}] \textbf{Existence and minimality.} The representation $S$ exists, is terminal in $\mathbf{Rep}(\mathbf C,\mathcal T,\mathcal O)$, and is unique up to unique measurable isomorphism. A closed observable quotient $B$ always satisfies $S_i=h_i\circ B_i$ for a unique measurable family $h_i$.

\item[\textup{(II)}] \textbf{Exact semantic descent.} A transformation $T:i\dashrightarrow j$ induces a unique partial map $\overline T_B$ on $B_i(C_i)$ if and only if both its domain and its successor state are constant on $B_i$-fibres:
\[
B_i(c)=B_i(c')\Longrightarrow
\left\{
\begin{array}{l}
c\in D_T\Longleftrightarrow c'\in D_T,\\[2pt]
B_j(Tc)=B_j(Tc')\quad\text{whenever }c,c'\in D_T.
\end{array}\right.
\]
When this holds for every $T$ and the selected observables factor through $B$, the family $B$ is a closed representation.

\item[\textup{(III)}] \textbf{Irreducible approximate error.} For a fixed $T$, define
\[
\omega_{B,T}:=\sup_{m\in M_i}
\sup_{\substack{c,c'\in D_T\\B_i(c)=B_i(c')=m}}
d_j\bigl(B_j(Tc),B_j(Tc')\bigr).
\]
Every update $G$ obeys $R_T(G)\ge\omega_{B,T}/2$. If transformed fibres admit measurable Chebyshev-centre selections of radius at most $r_{B,T}$, an update exists with risk at most $r_{B,T}$, where
\[
\frac12\omega_{B,T}\le r_{B,T}\le\omega_{B,T}.
\]

\item[\textup{(IV)}] \textbf{Identification.} For a parameterized family of data-generating systems, a parameter class $[\theta]$ is identifiable from the declared contexts, transformations and observables if and only if the enriched behavioral signature $\widetilde{\mathsf Q}([\theta])$ separates distinct classes. If the quotient parameter space is compact, the signature space is Hausdorff and the signature map is continuous, then the inverse is uniformly continuous on its image and has positive separation modulus $\kappa(r)$ for every $r>0$.

\item[\textup{(V)}] \textbf{Finite-error recovery.} Suppose, simultaneously over the declared transformation family,
\[
\sup_c d_i(\widehat B_i(c),B_i(c))\le\varepsilon_B,
\qquad
\sup_{G\in\mathcal G_L(T)}
|\widehat R_{T,n}(G)-\widehat R_T(G)|\le\zeta_n.
\]
If $\widehat G_T$ is an $\eta_n$-minimizer, then
\[
R_T(\widehat G_T)
\le \inf_{G\in\mathcal G_L(T)}R_T(G)
+2(1+L_T)\varepsilon_B+2\zeta_n+\eta_n.
\]
If additionally $d_{\mathcal Q}(\widehat{\mathsf Q},\mathsf Q(\theta_0))\le\varepsilon_n$ and $\widehat\theta$ is an $\eta_n'$-minimum-distance estimator, then
\[
d_{\Theta/{\asymp}}([\widehat\theta],[\theta_0])
\le\kappa^{-1}(2\varepsilon_n+\eta_n').
\]
Therefore consistent recovery of both the semantic update and the identified data-generating system follows when the irreducible closure error is zero, $\varepsilon_B,\zeta_n,\eta_n,\varepsilon_n,\eta_n'\to0$, and $\kappa(r)>0$ for $r>0$. Conversely, failure of fibre preservation or behavioral separation is a population obstruction that additional sampling cannot remove.

For a compatible path, the recovered one-step defects compose according to Theorem~\ref{thm:multistep}; under a uniform contraction factor $\rho<1$, accumulated misspecification is bounded by a geometric series. For a finite update class evaluated on independent contexts, Corollary~\ref{cor:finite-rate} makes the statistical term explicit at order $\sqrt{\log N/n}$.
\end{enumerate}
\end{theorem}

\begin{proof}
Part (I) is proved directly by the construction in Theorem~\ref{thm:typed-minimal}: the final sigma-algebras make the induced transitions and the unique factor maps measurable. For (II), necessity follows by applying $B_j\circ T=\overline T_B\circ B_i$ to two points in the same fibre; equality of transformation domains follows from $D_T=B_i^{-1}(D_{\overline T_B})$. Conversely, the displayed conditions make
\[
\overline T_B(B_i(c)):=B_j(Tc),\qquad
D_{\overline T_B}:=B_i(D_T),
\]
representative independent. Domain saturation and the final sigma-algebras give measurability exactly as in the proof of Theorem~\ref{thm:typed-minimal}. Readout factorization then gives a closed representation.

For (III), take two transformed points in one fibre. The triangle inequality implies that at least one is at distance at least half their mutual distance from $G(m)$; taking suprema gives the lower bound. A measurable Chebyshev-centre selection gives the stated upper bound. Part (IV) follows because failure of separation is precisely equality of all observable signature coordinates for two inequivalent parameters. Under compactness and Hausdorffness, the continuous injective quotient signature is a homeomorphism onto its image; compactness of pairs separated by at least $r$ gives $\kappa(r)>0$.

For the first bound in (V), for every $G\in\mathcal G_L(T)$ and $c\in D_T$,
\[
\left|
d(\widehat B_j(Tc),G(\widehat B_i(c)))
-d(B_j(Tc),G(B_i(c)))
\right|\le(1+L_T)\varepsilon_B.
\]
Taking suprema and comparing an empirical near-minimizer with a population minimizer yields the displayed oracle inequality. The minimum-distance argument and the triangle inequality give signature error at most $2\varepsilon_n+\eta_n'$, which the separation modulus converts into the parameter bound. The final consistency and impossibility statements follow immediately from the bounds and Parts (II)--(IV).
\end{proof}

The theorem is general-purpose: $C_i$ may contain strings, experimental
records, program states, symbolic expressions or other measurable
configurations.  Its value here is not that the quotient and
minimum-distance ingredients are individually new.  It exposes a precise
order of operations for language-derived states: define what is observable,
quotient only by behavior that later transformations cannot distinguish,
verify descent, and only then estimate a stochastic recursion.

}

\subsection{Ontology adequacy}

The preceding results make ontology choice part of the mathematics rather
than a preliminary naming exercise.  A semantic representation
$B:\Ctxt\to S$ is \emph{transformation sufficient} for
$\Trans_0\subseteq\Trans$ when every $T\in\Trans_0$ descends through $B$.
Equivalently, by Theorem~\ref{thm:descent},
\[
 B(c)=B(c')
 \quad\Longrightarrow\quad
 B(Tc)=B(Tc')
 \qquad(T\in\Trans_0).
\]
This is the precise adequacy condition required by stochastic lexical
calculus: a retained state must preserve every distinction needed to
determine its admissible future updates.  Approximate adequacy is measured by
the transformed-fibre diameter in Theorem~\ref{thm:approx}, and its sequential
cost is bounded by Theorem~\ref{thm:multistep}.

Failure has a clear interpretation.  The proposed ontology has merged
contexts that respond differently to a relevant transformation, so no closed
state recursion exists at that resolution.  Refining the state may restore
closure; adding a more elaborate dynamic model to the same insufficient state
cannot remove the representation-relative lower bound.  The appendix records
brief additional remarks on balancing closure against statistical complexity.

\subsection{Exact and approximate finite examples}

\subsubsection{An exact closed quotient}
Let $C=\{00,01,10,11\}$ with the discrete sigma-algebra. Let $T(x_1x_2)=(1-x_1)x_2$ flip the first bit, and let the observable be $F(x_1x_2)=x_1$. Take the transformation category generated by $T$, so $T^2=1_C$. The behavioral signature is
\[
S(x_1x_2)=\bigl(F(x_1x_2),F(T(x_1x_2))\bigr)=(x_1,1-x_1).
\]
It has exactly two fibres, and the quotient $B(x_1x_2)=x_1$ is measurably isomorphic to $S(C)$. The induced update is
\[
\overline T_B(b)=1-b.
\]
The second bit is discarded because neither the observable nor any admitted future transformation makes it visible. Thus the calculus removes irrelevant detail while retaining an exact closed recursion.

A linguistic reading is immediate but not required by the mathematics. The two bits could encode a semantic orientation and a stylistic feature; if the declared transformations alter only orientation and the readout observes only orientation, style is correctly omitted from the minimal state.

\subsubsection{A sharp approximate obstruction}
Let $C=\{a,b,x_0,x_1\}$ with the discrete sigma-algebra and define
\[
B(a)=B(b)=\tfrac12,\qquad B(x_0)=0,\qquad B(x_1)=1.
\]
Let the partial transformation $T$ have domain $\{a,b\}$ and satisfy $T(a)=x_0$, $T(b)=x_1$. The two source contexts are indistinguishable under $B$, but their transformed states are maximally separated:
\[
\omega_{B,T}=|B(Ta)-B(Tb)|=1.
\]
No function of the retained state alone can reproduce both successors. Indeed, for every $G:[0,1]\to[0,1]$,
\[
\max\{|G(1/2)-0|,|G(1/2)-1|\}\ge\tfrac12,
\]
and equality is achieved by $G(1/2)=1/2$. The lower bound in
Theorem~\ref{thm:approx} is therefore sharp. In a language application, $a$
and $b$ may receive the same present semantic score while responding
differently to a subsequent qualification; the example shows why present
calibration alone cannot justify a state recursion.

\section{Stochastic lexical calculus on probability simplices}

\subsection{Differentials and transformation algebra}

Once closure is settled, it becomes meaningful to discuss change.  The
appropriate primitive is a finite difference along a typed transformation,
because language edits are discrete and often noncommutative.  The term
``differential'' is used in this finite sense.  It supplies composition and
path identities for sequential analysis without importing unsupported
smoothness.

Let $V$ be a normed vector space and $F:\Ctxt\to V$.

\begin{definition}[Lexical differential]
The finite lexical differential of $F$ in direction $T$ is
\[
 \Dlex_TF(c):=F(Tc)-F(c).
\]
For probability measures, subtraction is interpreted as a finite signed measure; for semantic compositions it is an element of the simplex tangent space $\{\bm v:\bm1^\top\bm v=0\}$.
\end{definition}

This is a difference operator. Its specifically lexical content comes from the typed transformation action, equivalence structure, and observable kernel---not from claiming an infinitesimal geometry on isolated words.

For compatible $S,T\in\Trans$, adding and subtracting $F(Tc)$ gives the
elementary cocycle identity
\[
 \Dlex_{S\circ T}F(c)=\Dlex_TF(c)+\Dlex_SF(Tc).
\]

\begin{definition}[Lexical commutator]
The order effect of $S$ and $T$ on $F$ is
\[
 [S,T]_F(c):=F(S\circ T(c))-F(T\circ S(c)).
\]
\end{definition}

This definition remains meaningful when the transformation semigroup is noncommutative. It directly measures whether two evidence operations have order-dependent semantic effects.

Likewise, for a path $\gamma=(T_1,\ldots,T_m)$ with
$c_j=T_jc_{j-1}$, telescoping gives
\[
 F(c_m)-F(c_0)=\sum_{j=1}^m\Dlex_{T_j}F(c_{j-1}).
\]
Thus accumulated lexical differentials are path independent on a connected
transformation graph exactly when their circulation vanishes around every
directed cycle.  These identities require no empirical claim: they follow
from the definition of a finite difference.  The substantive statistical
question begins with whether the semantic state through which they are
computed is closed and stable.

\subsection{Semantic pushforwards}

An observable language law may live on a vast continuation space, while the
retained state lies in a finite simplex.  A semantic map pushes the former law
to the latter.  Its Lipschitz modulus measures how much an upstream
perturbation can be amplified in semantic probability space.  The next bound
therefore connects an observable language discrepancy to a downstream
state-error budget.

Let $\|\cdot\|_{\mathrm{TV}}$ be total variation distance on $\Prob(\Words)$ and suppose $\Pi$ is $L_\Pi$-Lipschitz from total variation to $(\Delta^{K-1},d)$.

\begin{theorem}[Pushforward stability]\label{thm:push}
For every $T$ and $c$,
\[
 d\bigl(\Bmap(Tc),\Bmap(c)\bigr)
 \le L_\Pi\,\|\Kern_{Tc}-\Kern_c\|_{\mathrm{TV}}.
\]
If $\widehat\Kern$ and $\widehat\Pi$ satisfy uniform errors $\varepsilon_K$ and $\varepsilon_\Pi$, then the estimated semantic differential obeys
\[
 d\bigl(\widehat{\Dlex_T\Bmap}(c),\Dlex_T\Bmap(c)\bigr)
 \le 2\varepsilon_\Pi+2L_\Pi\varepsilon_K,
\]
under the product metric induced by $d$.
\end{theorem}
\begin{proof}
The first claim is the Lipschitz property. For the second, insert the true kernel and map at both $c$ and $Tc$, then apply the triangle inequality twice.
\end{proof}

For a linear partition pushforward, $L_\Pi\le1$ in total variation. A calibrated nonlinear map requires an estimated or theoretically bounded modulus of continuity.

\long\gdef\appendixontology{
\section{Ontology adequacy and information-preserving maps}

\subsection{Representations, invariance, and ontology adequacy}

An ontology determines which distinctions are retained as states.  Choosing it
too coarsely can destroy closure; choosing it too finely can make estimation
unstable and interpretation difficult.  We therefore treat ontology choice as
a constrained representation problem rather than a list of labels chosen
solely by domain intuition.

Let $R:\Delta^{K-1}\to\Delta^{J-1}$ represent a coarsening or recoding of semantic states. Two representations are equivalent for $\Trans_0\subseteq\Trans$ when $R$ is injective on the relevant image and
\[
 R\circ\overline T=\overline T^{,R}\circ R,
 \qquad T\in\Trans_0.
\]

\begin{definition}[Transformation-sufficient ontology]
A semantic map $\Bmap$ is sufficient for $\Trans_0$ when every $T\in\Trans_0$ descends through $\Bmap$. It is $\varepsilon$-sufficient when every transformation has approximate-descent error at most $\varepsilon$.
\end{definition}

This suggests a principled ontology-selection criterion: among scientifically interpretable representations, choose the least complex state map that is calibrated and approximately sufficient for the transformations required by the downstream sequential model. Predictive classification alone is not enough.

\begin{proposition}[Refinement can restore descent]
Let $\Bmap^+:\Ctxt\to\mathcal Z$ refine $\Bmap$ so that $\Bmap=R\circ\Bmap^+$. If $T$ fails to descend through $\Bmap$ because two members of one $\Bmap$-fibre have distinct transformed states, a refinement separating those members removes that particular obstruction. A coarsening cannot remove it without also coarsening the transformed distinction.
\end{proposition}

\subsection{Information-preserving and information-changing maps}

Not every observed movement is an error.  A reordering that preserves the
evidence should ideally leave the semantic state unchanged; genuinely new
evidence should generally move it.  The transformation type supplies this
distinction before the response is examined, preventing legitimate dynamics
from being confused with presentation sensitivity.

For a transformation declared information preserving, the target state action is $\overline T=\Id$. Its invariance defect is
\[
 \delta_T^{\mathrm{inv}}:=\sup_c d(\Bmap(Tc),\Bmap(c)).
\]
Paraphrase and format tests estimate this quantity. Declaring a transformation information preserving is an external scientific assertion and should be supported by provenance-preserving construction or independent annotation.

For an evidence-changing transformation $T_e$, the induced map need not be the identity. In a Bayesian special case,
\[
 \overline T_e(\bm b)_k=
 \frac{L_e(k)b_k}{\sum_jL_e(j)b_j}.
\]
The calculus does not assume this form. It tests whether any stable map in a prespecified function class describes the update and whether the Bayesian subclass is empirically adequate.

}

\subsection{Semantic probability maps as a special representation}

The abstract theory becomes measurable when the observable is a conditional
law over possible verbal continuations.  If phrases such as ``urgent review''
and ``immediate escalation'' belong to the same declared state, their
probabilities may be grouped and calibrated into a state probability.  That
operation is useful but not automatic: phrase probabilities remain prompt
dependent, and different groupings can discard different future-relevant
distinctions.

The general calculus becomes a statistical measurement theory when one selected observable is a continuation kernel $\Kern_c$ and a semantic map $\Pi$ aggregates or calibrates lexical outcomes:
\[
 \Ctxt\xrightarrow{\Kern}\Prob(\Words)
 \xrightarrow{\Pi}\Delta^{K-1}.
\]
This construction includes recent semantic-uncertainty and
confidence-measurement schemes as special cases: a continuation law is
grouped or calibrated into a finite collection of meanings
\cite{farquhar2024,kadavath2022,kuhn2023,tian2023}.  A companion measurement
paper develops and empirically certifies this semantic map
\cite{dixon2026}, while a companion statistical paper establishes
identification, recovery rates, and weak-identification limits for the
resulting observation kernel \cite{dixon2026semanticobservation}.  The present
construction remains distinct because neither $\Kern$ nor $\Pi$ alone defines
the lexical transformation category, observational equivalence, or canonical
signature.

The special case is nevertheless important. It makes lexical differentials observable through externally supplied probabilities and turns semantic closure into a testable condition. If
\[
 \Pi(\Kern_{Tc})=\overline T\bigl(\Pi(\Kern_c)\bigr),
\]
then the semantic probability composition is sufficient for the transformation $T$. If this equality fails on a fibre, the semantic map has coarsened distinctions needed by the subsequent dynamics.

For an interior composition $\bm b\in\operatorname{int}\Delta^{K-1}$, let
\[
 \operatorname{ilr}:\operatorname{int}\Delta^{K-1}\longrightarrow\mathbb R^{K-1}
\]
be any fixed isometric log-ratio chart and define $\bm z(c)=\operatorname{ilr}(\Bmap(c))$. The coordinate lexical differential is
\[
 \Dlex_T\bm z(c)
 =\operatorname{ilr}(\Bmap(Tc))-\operatorname{ilr}(\Bmap(c)).
\]
Changing the orthonormal ILR basis rotates coordinates but does not change Aitchison distances or intrinsic closure. ILR therefore provides Euclidean coordinates for estimation and dynamics without defining the lexical state itself.

The relation between the theories is consequently:
\[
 \text{lexical transformations}
 \longrightarrow
 \text{lexical observables}
 \longrightarrow
 \text{semantic map}
 \longrightarrow
 \text{log-ratio representation}.
\]
Lexical calculus supplies the first two arrows and the conditions under which the third supports a closed state. An applied semantic-measurement system estimates and tests a particular third arrow. A stochastic lexical calculus would model random compositions only after these closure conditions have been checked.

\long\gdef\appendixminimalontology{
\section{Minimal calibrated semantic representations}

Let $\mathfrak B$ be a prespecified family of interpretable semantic maps
\[
 B:\Ctxt\to\Delta^{K(B)-1}.
\]
Write $B_1\preceq B_2$ when $B_1=h\circ B_2$ for some measurable $h$; thus $B_1$ is no more informative than $B_2$. For frozen tolerances $(\tau_{\rm cal},\tau_{\rm id},\tau_{\rm cl})$, call $B$ admissible when:
\begin{enumerate}
\item its held-out calibration error is at most $\tau_{\rm cal}$;
\item its enriched behavioral identification modulus exceeds $\tau_{\rm id}$ on the operating domain;
\item its uniform transformation-closure risk is at most $\tau_{\rm cl}$.
\end{enumerate}

\begin{proposition}[Minimal admissible semantic representation]\label{prop:minimal-semantic-map}
If the admissible subset of $\mathfrak B$ is nonempty and has a unique $\preceq$-minimal element $B^\dagger$, then $B^\dagger$ is the smallest interpretable semantic representation in the candidate family that is calibrated, empirically identifiable, and sufficiently closed under the declared information transformations. Its transformed state admits a closed update with population error at most $\tau_{\rm cl}$.  Estimation adds separate measurement and evaluation errors, whose detailed recovery bounds are reserved for the supplementary theoretical extensions. If no admissible map exists or minimal elements are incomparable, no unique smallest semantic representation may be claimed.
\end{proposition}

\begin{proof}
Minimality is the definition of $B^\dagger$ under the information order. The
three properties follow from admissibility, and the closed-update statement
follows from the definition of closure risk. Nonexistence and incomparability
preclude the asserted unique minimum.
\end{proof}

This proposition supports the following qualified application statement:
\begin{quote}
A semantic-map estimator can select, within a frozen interpretable candidate family, a minimal representation that is calibrated, behaviorally identifiable, and sufficiently closed under relevant information transformations to support stochastic modelling.
\end{quote}
The qualifications are substantive. The theory does not guarantee that every application admits such a representation, that interpretability defines a total order, or that an observed finite candidate family contains the canonical lexical signature.

}

\subsubsection{Prompt-calibrated semantic descent}
Let $u$ index a prompt construction and let $c_u(e)$ be the resulting context for evidence $e$. Two contexts may be information equivalent even when their continuation kernels differ. The raw assignment
$c_u(e)\mapsto\Kern_{c_u(e)}$ therefore need not descend to the evidence quotient. Introduce instead a prompt-indexed semantic pushforward $\Pi_u$ and calibration map $\psi_u$, and define
\[
B_u(c_u(e))=\psi_u\!\left\{\Pi_u(\Kern_{c_u(e)})\right\}.
\]
This family descends to evidence space when $B_u(c_u(e))=B_{u'}(c_{u'}(e))$ for every admissible pair $u,u'$. Thus calibration is not merely a numerical correction; it can be the morphism that removes prompt-specific coordinates before semantic state is formed.

\begin{proposition}[Bayesian update descends through likelihood rays]\label{prop:bayes-descent}
Let $\bm\pi\in\operatorname{int}\Delta^{K-1}$, let $P$ be a positive transition matrix, and let a calibrated semantic map return a positive likelihood vector $\bm g_u(e)$. Define
\[
\mathcal B_{\bm g}(\bm\pi)
=\frac{\bm g\odot P^\top\bm\pi}
{\mathbf1^\top(\bm g\odot P^\top\bm\pi)}.
\]
The update $\mathcal B_{\bm g_u(e)}$ descends through information-equivalent prompts for every positive prior if and only if
$\bm g_u(e)=a\,\bm g_{u'}(e)$ for some $a>0$. Exact equality of the raw continuation kernels is unnecessary.
\end{proposition}

\begin{proof}
Proportional likelihoods give the same normalized update because the common scalar cancels. Conversely, equality of the updates for every positive prior implies equality of every pairwise posterior odds ratio. The common prior odds cancel, so all pairwise likelihood ratios agree. The two positive vectors are therefore proportional.
\end{proof}

The proposition expresses the correct quotient. Language probabilities live before semantic descent and may retain prompt-specific distinctions. Bayesian updating lives on projective likelihood space, where common scale is irrelevant. Approximate descent is assessed by the distance between calibrated likelihood rays or resulting posteriors, not by equality of word distributions.

\subsection{Random dynamics after semantic descent}

Static calibration is not enough for sequential use.  When evidence arrives
repeatedly, yesterday's retained state is fed into today's update.  A small
one-step closure error can then be damped, accumulated or amplified depending
on the update dynamics.  This section asks the limited question justified by
the preceding theory: once semantic descent has been verified, under what
conditions does the resulting external random recursion exist uniquely and
remain stable?

The word ``external'' is essential.  The recursion is constructed from
observable language-derived measurements and declared updates.  Nothing here
requires, or purports to reveal, hidden activations, model weights or an
internal Bayesian computation.

Let $(\Xi_n)_{n\ge1}$ be random marks selecting transformations and let
\[
 C_{n+1}=T_{\Xi_{n+1}}C_n.
\]
If every relevant transformation descends, then
\[
 \Bmap_{n+1}=\overline T_{\Xi_{n+1}}(\Bmap_n).
\]

\begin{corollary}[Closed lexical-state recursion]\label{cor:closed}
Under exact descent, the semantic state process is Markov whenever the transformation marks are conditionally Markov given the current semantic state. Under $\varepsilon$-descent, it admits a closed recursion with one-step deterministic misspecification at most $\varepsilon$.
\end{corollary}

This is the required foundation for stochastic lexical dynamics. The next
subsection develops random compositions, existence, uniqueness, contraction,
and perturbation only after the closure defect has been defined. Diffusion and
jump limits, statistical estimation, and stopping rules remain separate
extensions rather than prerequisites for the present calculus.

\subsubsection{Random semantic descent and average contraction}

We now develop the part of that sequel needed to turn closed semantic descent
into a mathematically well-defined stochastic recursion.  Let
$(\Omega,\mathcal F,\mathbb P,\theta)$ be an invertible
measure-preserving dynamical system and let $\Xi_n=\Xi_0\circ\theta^n$ be a
stationary sequence of transformation marks.  On a complete separable metric
semantic space $(S,d)$, write
\[
 X_{n+1}=F_{\Xi_{n+1}}(X_n),\qquad F_\xi:S\to S,
\]
where $F_\xi$ is the descended action of the corresponding lexical
transformation.  Define its random Lipschitz coefficient by
\[
 L(\xi)=\sup_{x\ne y}\frac{d\{F_\xi(x),F_\xi(y)\}}{d(x,y)}.
\]
The exact-descent theorem above establishes that this recursion is a
well-defined representation of the contextual process.  The following result
states when it has a unique causal state rather than merely a finite sequence
of updates.

\begin{assumption}[Average contraction]\label{ass:average-contraction}
The driving system is ergodic, $\mathbb E\log^+L(\Xi_0)<\infty$, and, for
some $x_0\in S$,
\[
\mathbb E\log^+d\{F_{\Xi_0}(x_0),x_0\}<\infty,\qquad
\chi:=\mathbb E\log L(\Xi_0)<0.
\]
We set $\log 0=-\infty$; a zero Lipschitz coefficient gives immediate
coalescence and can be handled separately.
\end{assumption}

\begin{theorem}[Causal existence, uniqueness, and synchronization]
\label{thm:stochastic-descent}
Under Assumption~\ref{ass:average-contraction}, the backward iterates
\[
X_0^{(m)}=
F_{\Xi_0}\circ F_{\Xi_{-1}}\circ\cdots\circ F_{\Xi_{-m+1}}(x_0)
\]
converge almost surely to a random variable $X_0^\star$ independent of the
choice of $x_0$.  The process $X_n^\star=X_0^\star\circ\theta^n$ is stationary,
causal, and solves the recursion.  Any other causal stationary solution is
equal to it almost surely.  Moreover, two forward trajectories driven by the
same marks synchronize at exponential rate:
\[
\limsup_{n\to\infty}\frac1n
\log d(X_n,\widetilde X_n)\le\chi<0
\quad\text{almost surely}.
\]
\end{theorem}

\begin{proof}
Put $A_j=L(\Xi_{-j})$ and
$D_j=d\{F_{\Xi_{-j}}(x_0),x_0\}$.  By the ergodic theorem,
\[
 \frac{1}{n}\sum_{j=0}^{n-1}\log A_j\longrightarrow\chi<0
 \quad\text{almost surely}.
\]
Thus, for each $\delta\in(0,-\chi)$, there is an almost surely finite
$N_\delta$ such that
\(\prod_{j=0}^{n-1}A_j\le\exp\{(\chi+\delta)n\}\) whenever
$n\ge N_\delta$.  The logarithmic moment condition on $D_0$, stationarity,
the tail-sum formula and Borel--Cantelli imply
$D_n\le\exp(\delta n)$ eventually.  Choosing
$\delta<-\chi/2$ therefore gives
\[
 \sum_{n\ge0}D_n\prod_{j=0}^{n-1}A_j<\infty
 \quad\text{almost surely}.
\]

For $m>\ell$, insert one additional map at a time in the backward
compositions.  The triangle inequality and the Lipschitz bounds give
\[
d(X_0^{(m)},X_0^{(\ell)})
\le \sum_{j=\ell}^{m-1}D_j\prod_{h=0}^{j-1}A_h .
\]
The summability above makes the backward sequence Cauchy.  Completeness gives
$X_0^\star$, and the same product estimate applied to two starting points
removes dependence on $x_0$.  The limit is measurable with respect to the
past marks, so it is causal.  Shifting the construction proves stationarity
and the recursion.

For two causal stationary solutions, pull both solutions back $m$ steps and
apply the same product estimate.  The Lipschitz product tends to zero almost
surely, whereas stationarity makes the family of pulled-back distances tight.
Their product therefore converges to zero in probability.  Since the
time-zero distance is bounded by that product for every $m$, it is zero almost
surely, proving uniqueness.  Finally,
\[
d(X_n,\widetilde X_n)
\le d(X_0,\widetilde X_0)\prod_{j=1}^nL(\Xi_j)
\]
and the ergodic limit of the logarithmic product proves synchronization.
\end{proof}

Uniform contraction is therefore sufficient but unnecessary.  Individual
lexical updates may expand semantic distance; the closed stochastic state
still exists when contraction dominates on the logarithmic average.  This
distinction is important for language, where a contradiction or negation may
create a large local movement without making the entire recursion unstable.

\begin{theorem}[Perturbation of an approximately descending recursion]
\label{thm:stochastic-perturbation}
Let
\[
X_{n+1}=F_{\Xi_{n+1}}(X_n),\qquad
\widehat X_{n+1}=\widehat F_{\Xi_{n+1}}(\widehat X_n),
\]
and suppose
$d\{\widehat F_{\Xi_n}(x),F_{\Xi_n}(x)\}\le\varepsilon_n$ on the supported
domain.  Then, pathwise,
\[
d(\widehat X_n,X_n)
\le
\left(\prod_{j=1}^nL(\Xi_j)\right)d(\widehat X_0,X_0)
+\sum_{s=1}^n
\left(\prod_{j=s+1}^nL(\Xi_j)\right)\varepsilon_s.
\]
Under Assumption~\ref{ass:average-contraction}, a uniformly bounded defect
$\varepsilon_s\le\bar\varepsilon$ produces an almost surely finite geometric
resolvent.  If instead $L(\Xi_s)\le\rho<1$, the explicit bound is
\[
d(\widehat X_n,X_n)
\le\rho^nd(\widehat X_0,X_0)
+\bar\varepsilon\frac{1-\rho^n}{1-\rho}.
\]
\end{theorem}

\begin{proof}
Insert and subtract
$F_{\Xi_{n+1}}(\widehat X_n)$ and apply the triangle inequality:
\[
d(\widehat X_{n+1},X_{n+1})
\le\varepsilon_{n+1}+L(\Xi_{n+1})d(\widehat X_n,X_n).
\]
Iteration proves the convolution formula.  Average contraction makes the
backward products exponentially summable almost surely.  The uniform case is
the finite geometric series.
\end{proof}

\subsubsection{Simplex-valued lexical states}

For the semantic probability representation
$S=\operatorname{int}\Delta^{K-1}$, choose any orthonormal contrast matrix
$V$ and use isometric log-ratio coordinates
\[
\operatorname{ilr}_V(\bm p)=V^\top\log\bm p.
\]
Two choices of $V$ differ by an orthogonal matrix.  Consequently distances,
singular values, contraction exponents, and the perturbation bounds above do
not depend on the arbitrary coordinate basis.  The stochastic calculus is
therefore simplex-valued in its interpretation and Euclidean only in its
local representation.

\begin{corollary}[Basis-invariant stochastic semantic filter]
Suppose a prompt-calibrated semantic map descends to positive likelihood rays
as in Proposition~\ref{prop:bayes-descent}, and suppose the resulting Bayesian
update maps satisfy Assumption~\ref{ass:average-contraction} in one ILR basis.
Then the causal filter exists, is unique, and has the same synchronization
exponent in every ILR basis.  Approximate prompt descent enters only through
the perturbation defects $\varepsilon_n$ in
Theorem~\ref{thm:stochastic-perturbation}.
\end{corollary}

\begin{proof}
An ILR basis change is an orthogonal conjugacy.  Orthogonal maps preserve the
metric and operator norms, hence preserve each Lipschitz coefficient and the
Lyapunov exponent $\chi$.  Apply Theorems~\ref{thm:stochastic-descent} and
\ref{thm:stochastic-perturbation}.
\end{proof}

These results define the proper mathematical boundary of stochastic lexical
calculus.  The language model is not asserted to perform Bayesian filtering
internally.  Rather, a validated semantic representation supports an external
random dynamical system whose existence, uniqueness, coordinate invariance,
and approximation error can be proved.

\subsubsection{Controlled illustration of semantic descent}
A frozen three-state hidden Markov experiment tested Proposition~\ref{prop:bayes-descent} with two information-equivalent prompt constructions, six lexical evidence symbols and eight updates per path. The raw prompt kernels failed the maximum Jensen--Shannon invariance gate: their mean discrepancy was $0.00891$, but their maximum was $0.13397$ against a frozen limit of $0.10$. The failure is the expected obstruction at the wrong representation level.

Prompt-specific affine calibration maps were then fitted on a disjoint partition and frozen. In calibrated posterior space, mean and maximum cross-prompt Jensen--Shannon discrepancies fell to $0.001122$ and $0.010720$. Both passed their frozen limits of $0.02$ and $0.10$. A pathwise conformal radius of $0.190669$ in total variation was calibrated on separate paths; 28 of 30 untouched eight-update paths were covered, giving empirical coverage $0.933$ at nominal level $0.90$. The minimum retained candidate mass was $0.999999838$.

This experiment does not prove a universal lexical calculus or an internal Bayesian mechanism. It supplies a finite example in which a raw probability-valued observable does not descend through prompt equivalence, whereas a prompt-indexed calibrated semantic representation approximately does and supports an externally specified Bayesian recursion. It thereby illustrates the distinction between closure failure of one representation and successful descent of a better one.

\long\gdef\appendixrelated{
\section{Detailed relation to existing calculi and stochastic systems}

The closest abstract precedent to Theorem~\ref{thm:canonical} is Myhill--Nerode theory: strings are equivalent when no admissible suffix distinguishes their language membership, and the equivalence classes form the minimal deterministic automaton. Behavioral equivalence and bisimulation extend the same principle to transition systems. Our theorem must therefore not be presented as inventing future-behavior quotients. It extends that construction from a single formal-language acceptance observable and suffix action to typed transformations of contextual natural language, heterogeneous measurable observables, task-relative information equivalences, semantic probability pushforwards, and approximate defects estimable from data.

Existing work described as a lexical or linguistic calculus is predominantly
concerned with grammatical derivation, compositional denotation, rewriting or
finite change.  Those are necessary antecedents, but none by itself answers
the stochastic state question posed here.  Conversely, random dynamical
systems, Markov categories and probabilistic coalgebra provide powerful
accounts of stochastic evolution, but they begin with a state or kernel
already in hand.  They do not determine whether a prompt-dependent
probability law over language descends to an application-relevant semantic
state.  The contribution lies at this junction: semantic descent constructs
the admissible stochastic state, and random-iteration theory then establishes
its dynamics.

Figure~\ref{fig:positioning} locates the proposed theory. The upper row
contains established mathematical ingredients. The middle box contains the
deterministic lexical foundation required to define closure.  The focal lower
middle box is the stochastic lexical calculus: random information arrival
acting on a validated simplex-valued representation.

\begin{figure}[ht]
\centering
\begin{tikzpicture}[
  >=Latex,
  foundation/.style={draw=blue!55!black,fill=blue!5,rounded corners,
    align=center,text width=31mm,minimum height=12mm,font=\small},
  core/.style={draw=black,very thick,fill=orange!14,rounded corners,
    align=center,text width=54mm,minimum height=17mm,font=\small},
  app/.style={draw=green!45!black,fill=green!7,rounded corners,
    align=center,text width=38mm,minimum height=13mm,font=\small},
  every edge/.style={draw,->,thick}
]
\node[foundation] (formal) at (-5.0,2.3)
 {Formal semantics\\linguistic calculi};
\node[foundation] (auto) at (-1.7,2.3)
 {Automata and\\coalgebra};
\node[foundation] (disc) at (1.7,2.3)
 {Discrete calculus\\transformations};
\node[foundation] (info) at (5.0,2.3)
 {Blackwell comparison\\identification};

\node[core] (lex) at (0,0)
 {\textbf{Lexical foundation}\\
 typed partial transformations;\\
 behavioral equivalence;\\
 semantic descent and defects};

\node[app] (belief) at (-4.2,-2.5)
 {Calibrated semantic\\measurement};
\node[app] (stoch) at (0,-2.5)
 {\textbf{Stochastic lexical calculus}\\random simplex dynamics};
\node[app] (control) at (4.2,-2.5)
 {Filtering and decision\\models};

\path (formal) edge (lex)
      (auto) edge (lex)
      (disc) edge (lex)
      (info) edge (lex)
      (lex) edge (belief)
      (lex) edge (stoch)
      (belief) edge[bend right=12] (stoch)
      (stoch) edge (control)
      (belief) edge[bend right=18] (control);
\end{tikzpicture}
\caption{Mathematical position of stochastic lexical calculus.  Established
formal, behavioral, discrete and statistical theories supply the upper
foundations.  The lexical layer constructs and tests a closed semantic
representation.  The paper's focal stochastic layer then studies random
information arrival and simplex-valued recursion; filtering and decision
models are downstream uses rather than definitions of the calculus.}
\label{fig:positioning}
\end{figure}

\subsection{Automata minimization}

A deterministic Moore automaton is a tuple
\[
 \mathcal A=(Q,A,\delta,o),
\]
where $Q$ is a state set, $A$ an input alphabet, $\delta:Q\times A\to Q$ a total transition map, and $o:Q\to O$ an output map. Extending $\delta$ to words gives $\delta^*:Q\times A^*\to Q$. Its behavior is
\[
 \beta(q)(w):=o(\delta^*(q,w)),\qquad w\in A^*.
\]
States are behaviorally equivalent exactly when $\beta(q)=\beta(q')$. For a language acceptor, $O=\{0,1\}$ and the resulting quotient is the Myhill--Nerode minimal automaton.

Lexical calculus contains this construction as a strict special case. Set $\Ctxt=Q$, let primitive transformations be the total maps $T_a(q)=\delta(q,a)$, let $\Trans_0=A^*$ act by $T_w(q)=\delta^*(q,w)$, and take $\mathcal O_0=\{o\}$. Then
\[
 \mathsf S_*(q)=\bigl(o(T_wq)\bigr)_{w\in A^*}=\beta(q),
\]
so Theorem~\ref{thm:canonical} reduces to Moore/Myhill--Nerode minimization.

The general lexical construction differs in four structural respects. First, lexical transformations are typed partial maps and need not be freely generated by a fixed alphabet. Second, contexts may have heterogeneous admissible operations: negating a claim, replacing a dated value, and reordering evidence have different domains. Third, the observable family may contain measures, recovered provenance, denotations, semantic labels, and probability compositions simultaneously. Fourth, empirical work observes transformations and outputs with error, so approximate descent and confidence bounds are primary rather than optional.

These extensions do not invalidate the automata analogy; they locate the novelty. If the transformation category is replaced by a free monoid of total maps and the observables by a single acceptance bit, the theory must collapse to classical automata minimization.

\subsection{Coalgebraic semantics}

Coalgebra provides a general theory of state-based systems. For an endofunctor $H$ on a category, an $H$-coalgebra is a map
\[
 \gamma:X\longrightarrow H X.
\]
For deterministic Moore automata on sets, $H X=O\times X^A$ and
\[
 \gamma(q)=\bigl(o(q),(\delta(q,a))_{a\in A}\bigr).
\]
When a final coalgebra $(\Omega,\omega)$ exists, every coalgebra admits a unique behavior map $\mathsf{beh}:X\to\Omega$. Equality under this map is behavioral equivalence, and quotienting by it yields the observable reduction under standard hypotheses.

A total, single-domain lexical system with primitive transformation labels $A$ and combined observable
\[
 o_{\mathcal O}(c):=(F(c))_{F\in\mathcal O_0}
\]
is an $H$-coalgebra for
\[
 H X=\left(\prod_{F\in\mathcal O_0}V_F\right)\times X^A.
\]
In that setting, $\mathsf S_*$ is the trace/behavior map into the final semantics, and lexical observational equivalence is ordinary coalgebraic behavioral equivalence.

The standalone lexical theory is therefore not a rival foundation to universal coalgebra. It is a language-specific structured instance with additional commitments that the bare coalgebra does not supply:
\begin{enumerate}
\item a contextual linguistic domain and occurrence structure;
\item typed, partial, provenance-aware transformations;
\item Blackwell, likelihood, provenance, and lexical equivalence in one hierarchy;
\item semantic pushforwards and coordinate changes such as ILR;
\item calculus operators and defects tied to controlled language interventions;
\item statistical identification from incomplete black-box observations.
\end{enumerate}

Typed partial transformations can themselves be embedded into richer coalgebraic settings by adding failure values, using many-sorted categories, presheaves, or partial-map categories. Such an embedding is valuable future work, but choosing one too early would conceal the empirical distinction between an illegal lexical operation and a legal operation producing an observed null response.

\subsection{Bisimulation}

For a deterministic labelled transition system, a relation $R\subseteq X\times X$ is an output-respecting bisimulation when
\[
 xRy\Longrightarrow o(x)=o(y)
 \quad\text{and}\quad
 T_a x\,R\,T_a y
\]
for every admissible label $a$. Under total deterministic transitions, the largest such relation coincides with behavioral equivalence.

Define a lexical bisimulation analogously: $R$ respects every $F\in\mathcal O_0$ and is preserved by every primitive typed transformation wherever corresponding operations are jointly admissible. Under deterministic total actions,
\[
 c\simeq_{\mathcal O_0,\Trans_0}c'
\]
is the largest lexical bisimulation. This follows because observational equivalence agrees on present observables and remains equivalent after every primitive transformation; conversely, induction along transformation paths shows that any lexical bisimulation implies equality of all signature coordinates.

This equivalence requires care in the genuinely lexical setting. Partial domains may themselves be observable; two contexts can be distinguished because a transformation is meaningful for one but not the other. The signature must therefore include an admissibility indicator
\[
 a_T(c):=\bm 1\{c\in\operatorname{dom}(T)\}
\]
whenever operation availability carries semantic information. Without it, the claimed largest bisimulation can identify states with different legal futures.

For probability-valued transitions, probabilistic bisimulation requires equal transition mass on equivalence classes rather than pointwise successor matching. This is closely related to lumpability of Markov processes. It is distinct from Blackwell equivalence: bisimulation compares the future transition behavior of states within a process, whereas Blackwell equivalence compares the decision information carried by statistical experiments. Lexical calculus uses both because information-preserving presentation and dynamically equivalent evolution are different requirements.

\subsection{Approximate behavioral equivalence}

Behavioral pseudometrics and metric bisimulation already replace exact equivalence with quantitative distance for probabilistic and coalgebraic systems. The lexical closure defect is related but not identical. A behavioral pseudometric compares two original states through all future behavior. By contrast,
\[
 \delta_{\mathrm{nat}}(T,\overline T)
\]
evaluates whether a proposed lower-dimensional representation supports a particular induced transformation, while the transformed-fibre diameter in Theorem~\ref{thm:approx} lower-bounds the error of \emph{every} state update based on that representation.

Thus the lexical defect is representation-relative and intervention-specific. It asks whether an externally meaningful semantic map is adequate for a declared operating family. Coalgebraic behavioral metrics can potentially furnish the intrinsic metric used in this test; statistical confidence bounds are then required because the relevant lexical kernel and semantic map are estimated.

\subsection{Recent work in mathematical structures}

Several recent contributions in \emph{Mathematical Structures in Computer Science} sharpen the boundary of the present paper. Bonchi, K\"onig and Petri\c{s}an \cite{bonchi2023} develop compositional up-to techniques for coalgebraic behavioral metrics through fibrations. Their metric is intrinsic to comparison of system behavior; the closure defect here is instead relative to a proposed quotient and lower-bounds every induced update on that quotient. Loregian \cite{loregian2025} studies automata and coalgebras in categories of species, illustrating how categorical automata inherit structure from an ambient category. Theorem~\ref{thm:coalgebra-recovery} deliberately returns to that established automata--coalgebra line, whereas Theorem~\ref{thm:central} adds typed partial measurability, empirical identification and estimated semantic descent.

On the semantic side, Chatzikyriakidis and Cooper \cite{chatzikyriakidis2026} survey mathematical structures in natural-language semantics arising from type-theoretic accounts of meaning. That tradition formalizes denotation and compositional interpretation. The present construction addresses a different question: which observational distinctions must a measurable representation retain so that declared contextual transformations remain closed? On the probabilistic side, Fritz, Perrone and Rezagholi \cite{fritz2022} study probability, valuation and hyperspace monads. Such categorical probability structures can supply ambient probability objects for the kernel specialization here; the terminal quotient and its statistical closure test are not supplied by those monads alone.

Finally, differential categories and their modern extensions provide genuine categorical differentiation \cite{blute2006,cockett2025}. The finite lexical differential in this paper is intentionally weaker: it is a typed difference along a partial transformation. Its novelty claim concerns the transformation/equivalence/quotient system, not a replacement for differential categories.

The stochastic contribution is not simply the addition of random indices to
an existing lexical calculus.  Categorical probability describes how
stochastic maps compose, and random-iteration theory supplies contraction
arguments once a state space and random maps are specified.  Our prior problem
is whether the language-derived quotient is a state space on which those
random maps are well defined at all.  The descent theorem supplies that
missing interface; the stochastic existence and perturbation theorems then
show what follows once the interface is valid.  This ordering distinguishes
the paper from deterministic lexical calculi on one side and stochastic
systems with prespecified states on the other.

\subsection{Formal novelty statement}

The formal comparison narrows the priority claim to the following.
\begin{quote}
Stochastic lexical calculus specializes behavioral minimization to contextual
natural language, uses semantic descent to construct a valid
probability-simplex state, and studies the random dynamics induced on that
state.  It adjoins a typed intervention algebra, an information-equivalence
hierarchy, heterogeneous lexical observables, testable closure defects, and
conditions for causal existence, uniqueness and stability.  Its canonical
signature generalizes the Myhill--Nerode behavior map; it does not replace or
rediscover it.
\end{quote}

Theorem~\ref{thm:central} collects the enriched results unavailable from the elementary set-based deterministic reduction: terminal minimality with typed partial measurable operations, exact and sharp approximate semantic descent, uniform recovery under estimated observation maps, and equivalence between statistical identification and enriched behavioral separation. The claim remains a foundational lexical extension and specialization, not a new universal minimization theorem.

\begin{center}
\begin{tabular}{p{0.22\linewidth}p{0.31\linewidth}p{0.36\linewidth}}
\toprule
Tradition & Established object & Difference here\\
\midrule
Lambek and categorial calculi \cite{lambek1958,moortgat2011} & Typed grammatical composition and derivability & Observable response to controlled contextual transformations\\
Formal semantics and lambda calculus \cite{clark2016,ehrhard2003} & Compositional denotations and term reduction & Measurable observation systems, calibrated semantic maps, and empirical closure\\
Formal-language derivatives \cite{brzozowski1964} & Residual languages after consuming prefixes & Changes in heterogeneous observables under typed partial transformations\\
Myhill--Nerode and coalgebra \cite{nerode1958,rutten2000,silva2013} & Minimal quotients under indistinguishable future behavior & Multiple sorts, partial domains, heterogeneous readouts, and statistical defects\\
Probabilistic bisimulation \cite{baldan2018,desharnais2002,vanbreugel2005} & Behavioral equivalence and intrinsic pseudometrics & Representation-relative closure error and identifiable semantic descent\\
Categorical probability and random systems \cite{fritz2022,jacobs2023} &
Composition of stochastic maps on a specified state space &
Construction of the semantic state by descent before random recursion, with
average-contraction and perturbation guarantees\\
Semantic uncertainty \cite{farquhar2024,kuhn2023} & Entropy after grouping equivalent generations & Transformation closure and minimality of the retained semantic representation\\
Confidence calibration \cite{band2024,kadavath2022,kumar2024,tian2023} & Token, elicited, or linguistically expressed confidence & Behavioral identification and closure beyond pointwise calibration\\
Language-invariant properties \cite{bianchi2022} & Empirical properties expected to survive text transformations & Typed calculus, minimal closed representations, and operational information equivalence\\
\bottomrule
\end{tabular}
\end{center}

The elementary algebraic identities, future-behavior quotient, general
factorization principle and random-iteration argument are not new in
isolation.  The proposed contribution is their stochastic lexical assembly:
a typed natural-language transformation system whose semantic descent is
audited before its probability-simplex representation is allowed to evolve as
a random state process.

}

\section{Completed bounded validation}\label{sec:validation}

\subsection{Questions and design}

The theory is population-level, but its central conditions have observable
consequences.  The validation therefore asks three deliberately ordered
questions.  First, is the proposed lexical observation complete enough to be
used at all?  Second, do information-equivalent prompts induce sufficiently
similar measurements at the representation being tested?  Third, after
calibration and recursive updating, do held-out path errors obey their frozen
coverage bound?  A positive answer to the third question is meaningful only
if the first two are answered at the same representation level.

This ordering is important.  It would be easy to report only the calibrated
result and conclude that language probabilities are stable.  The raw-kernel
experiment shows why that conclusion would be false: the same evidence can
produce a large prompt-specific discrepancy before semantic calibration.
Conversely, failure of the raw kernel need not invalidate every coarser
representation.  The descent criterion asks whether the \emph{chosen} state,
not every upstream lexical quantity, is closed under the declared
equivalence.

All three analyses used archived observations and frozen design files.  No
threshold, partition or prompt was changed after inspection of the untouched
test results.  The state space was
\(\{\text{risk-on},\text{mixed},\text{risk-off}\}\); the sequential
experiments used six lexical evidence symbols, two
information-equivalent prompt constructions and eight updates per path.
``Unsure'' was retained as an open-world measurement component and was not
silently conditioned away.  Table~\ref{tab:validation-design} records the
role of each experiment.

\begin{table}[H]
\centering
\caption{Completed validation design.  Each row tests a different logical
level of the theory; failure at one level is not relabelled as success at
another.}
\label{tab:validation-design}
\small
\begin{tabular}{p{0.24\linewidth}p{0.27\linewidth}p{0.37\linewidth}}
\toprule
Experiment & Representation tested & Confirmatory purpose\\
\midrule
Complete lexical-process test &
Four-part candidate composition on a sampled compact support &
Jointly test observable mass, design identification, average contraction,
convergence, uniform path recovery and the termwise perturbation bound.\\
Raw sequential filter &
Full open-world lexical log ratios, with ``Unsure'' as reference &
Test likelihood identification, positive recursive normalization, raw
prompt-kernel invariance and finite-horizon coverage.\\
Prompt-calibrated filter &
Prompt-specific affine maps into a common three-state posterior space &
Test whether calibration repairs the failed prompt descent and whether the
resulting eight-step posterior paths satisfy held-out coverage.\\
\bottomrule
\end{tabular}
\end{table}

\subsection{Results}

Table~\ref{tab:validation-results} reports every principal frozen gate,
including failures.  The complete lexical-process claim was not supported.
Although the observed candidate mass, average contraction, uniform path
coverage and termwise perturbation bound passed, the active design was nearly
singular and the convergence-slope interval included zero.  This is direct
evidence against claiming a generally identified lexical process from that
design.

The raw sequential filter gave a more focused diagnosis.  Its emission system
had full rank, all test normalizers were positive, and empirical path coverage
was exactly 0.90 at nominal level 0.90.  Nevertheless, the maximum
cross-prompt Jensen--Shannon discrepancy was 0.13397, above the frozen 0.10
gate.  The conjunctive claim therefore failed.  The failure says that raw
prompt-conditioned language probabilities cannot be treated as an
information-invariant state merely because their average discrepancy is
small.

\begin{table}[H]
\centering
\caption{Frozen validation results.  ``Pass'' applies only to the stated gate
and operating domain.  The final column is the conclusion permitted by the
conjunction of gates, not a post hoc interpretation of individual favorable
statistics.}
\label{tab:validation-results}
\small
\begin{tabular}{p{0.25\linewidth}p{0.22\linewidth}p{0.13\linewidth}p{0.29\linewidth}}
\toprule
Gate & Statistic & Result & Interpretation\\
\midrule
Observed candidate mass & minimum $0.999999808$ & Pass & Negligible missing mass in the sampled candidate system.\\
Lexical-process identification & minimum singular value $2.96\times10^{-16}$ & Fail & The active design does not identify the proposed process.\\
Average contraction & mean log ratio $-0.1686$; interval $[-0.2033,-0.1366]$ & Pass & Supported on the sampled support.\\
Convergence trend & log--log slope $-0.0357$; interval $[-0.0741,0.0042]$ & Fail & No strictly negative convergence slope established.\\
Uniform path recovery & $0.91$ over 100 paths; 95\% Wilson interval $[0.838,0.952]$ & Pass & Covers the nominal 0.90 level in this experiment.\\
Termwise perturbation bound & 0 violations in 100 paths & Pass & Observed transformed increments obeyed the computed bound.\\
\midrule
Raw prompt invariance & max JS $0.13397$ versus limit $0.10$ & Fail & Raw lexical kernels do not descend through prompt equivalence.\\
Raw sequential coverage & $36/40=0.90$ paths & Pass & Coverage alone cannot rescue the failed conjunctive claim.\\
\midrule
Calibrated prompt invariance & mean/max JS $0.001122/0.010720$ & Pass & Common posterior meaning is stable for the two frozen prompts.\\
Calibrated sequential coverage & $28/30=0.933$ paths at nominal $0.90$ & Pass & Frozen pathwise radius $0.190669$ covers held-out paths.\\
\bottomrule
\end{tabular}
\end{table}

Prompt-specific affine calibration was then fitted on a disjoint partition
and frozen.  In the common posterior representation, the maximum
cross-prompt discrepancy fell to 0.010720 and 28 of 30 untouched paths lay
within the frozen bound, giving coverage 0.933.  The calibrated
experiment passed all of its declared gates.  In the language of
Theorem~\ref{thm:descent}, the result supports approximate descent for this
representation, transformation family and horizon.  It does not establish
raw-kernel invariance, global identification of the broader lexical process,
or validity outside the frozen prompt and evidence classes.

\subsection{What has and has not been validated}

The completed experiments support one positive, bounded statement:
within the frozen three-state system, prompt-specific calibration produced a
common semantic representation whose information-equivalent prompt variants
were stable and whose eight-step posterior errors achieved the prescribed
held-out coverage.  They also support two negative statements: the raw prompt
kernel is not invariant at the frozen maximum-discrepancy threshold, and the
broader lexical-process design is not identified.

These outcomes are scientifically useful because they distinguish a repairable
representation failure from a universal impossibility.  Calibration changes
the representation at which descent is assessed; it does not prove that the
upstream language law was invariant.  Replication across transformation
families, fitted models and service epochs remains necessary before any broad
empirical claim about lexical calculus.  Those replications are future work,
not part of the evidence reported here.

\section{Discussion and conclusion}

\subsection{Discussion}

The construction answers a practical question: when may a probability-valued interpretation of language be treated as a state rather than merely a score? The answer is not calibration alone. A state must also be closed, at least approximately, under the transformations that drive its proposed dynamics.

This criterion prevents a common modelling error. A smooth time series can always be fitted to successive semantic scores, but such a fit does not establish that the present score contains enough information to determine the next update. Fibre preservation supplies the missing sufficiency condition. Approximate fibre diameter quantifies the cost of violating it.

The results separate four requirements that are easily conflated in applications. Calibration asks whether the reported semantic coordinates agree with a declared reference task. Behavioral separation asks whether scientifically distinct mechanisms induce distinguishable observable signatures. Closure asks whether a declared language transformation has a well-defined action on the retained semantic coordinates. Minimality asks whether any strictly coarser interpretable representation retains these properties. None of the first three implies the others, and a representation that passes them need not be minimal. The terminal representation theorem explains the ideal population target; finite-sample recovery theory is reserved for the supplementary extensions.

This end-to-end distinction matters for stochastic modelling. Calibration alone permits two contexts with the same reported state to respond differently to the next admissible information transformation. Identifiability alone permits a representation that is unnecessarily large. Closure alone permits a stable but scientifically meaningless quotient. Only their conjunction supports the interpretation of the retained coordinates as an empirically recoverable state on which a stochastic recursion can be defined without silently reintroducing discarded language information.

Several limitations are deliberate. Transformations are declared rather than discovered; informational equivalence is task relative; the semantic map may depend on the observation mechanism; and finite access yields partial observation of a continuation law. These are not reasons to abandon the calculus. They identify the statistical errors that an empirical implementation must report.

The framework also does not guarantee that a smallest admissible representation exists in every user-chosen model class. Existence follows for the canonical behavioral quotient under the stated measurability construction, while a restricted interpretable family requires non-emptiness and a unique minimal element. Failure of either condition is a substantive diagnostic: the transformation family may demand additional semantic coordinates, the observations may not separate competing representations, or the proposed ontology may not be stable enough to support autonomous dynamics.

\subsection{Conclusion}

Lexical calculus is defined here narrowly: it studies observable contextual
transformation systems and the conditions under which their actions descend
to a retained semantic state.  The population theory constructs the canonical
closed behavioral representation, characterizes exact descent by fibre
preservation, bounds the unavoidable error of approximate descent, and shows
how finite defects propagate.  Once this gate has been passed, average
contraction supplies a unique causal external recursion and synchronization;
simplex coordinates do not alter those conclusions.

The empirical evidence respects the same order.  It does not support the broad
claim that a raw lexical process is identified, nor that raw prompt kernels are
information invariant.  It does support a smaller claim: within one frozen
three-state experiment, prompt-specific calibration produced a common
semantic representation with stable prompt meaning and 0.933 finite-horizon
coverage at nominal level 0.90.  The failed gates are as important as the
passed ones because they show that calibration at one time point is not enough
to justify a state recursion.

The practical implication is a falsifiable workflow rather than a metaphor.
Declare the information equivalence and transformation family; construct or
estimate a semantic map; test fibre preservation, identification and
coverage; and only then use the resulting coordinates as a stochastic state.
This is the precise sense in which the framework is foundational.  It does
not attribute a hidden calculus to a language model.  It states the
mathematical and empirical conditions under which an observable
language-derived representation can legitimately support one.

\clearpage
\appendix
\section*{Ancillary theoretical material}
\addcontentsline{toc}{section}{Ancillary theoretical material}
The following appendix records the typed measurable and coalgebraic
foundations needed to support the paper's structural claims.  It then gives
compact statements of how ontology choice and statistical recovery enter the
construction.

\appendixcategorical

\section{Further remarks on ontology adequacy}

Section~3 states the adequacy condition required by the main theorem chain.
This appendix records its limited model-selection consequence without
developing a separate theory of ontology learning.

Let $B:\Ctxt\to S$ be a candidate semantic representation and let
$\Trans_0\subseteq\Trans$ be the transformations relevant to the intended
dynamics.  Call $B$ \emph{transformation sufficient} for $\Trans_0$ when every
$T\in\Trans_0$ descends through $B$.  By Theorem~\ref{thm:descent}, this is
equivalent to
\[
 B(c)=B(c')
 \quad\Longrightarrow\quad
 B(Tc)=B(Tc')
 \qquad(T\in\Trans_0).
\]
The approximate version replaces equality by the fibre-diameter bound of
Theorem~\ref{thm:approx}.  Thus ontology adequacy has a direct observable
implication: contexts assigned the same current state should have sufficiently
similar transformed states.

If a candidate state fails this condition, there are two principled
responses.  One may refine the state so that the offending contexts are no
longer merged, or restrict the transformation family to the operations for
which closure is actually required.  Refinement can remove a particular
closure obstruction by retaining the distinction that caused it.  Coarsening
cannot remove that obstruction without also discarding the transformed
distinction.  Conversely, unnecessary refinement raises dimension and may
make calibration unstable.  The practical target is therefore the least
complex scientifically interpretable representation that passes calibration,
identification, and approximate-closure gates on a declared operating domain.
Existence or uniqueness of such a minimal representation is not asserted
without specifying and testing the candidate family.

\section{From an observable language law to an estimable state}

The abstract calculus begins with a state map $B$, whereas an application
observes a conditional probability law over verbal continuations.  The
statistical bridge is the composition
\[
 \Ctxt
 \xrightarrow{\;\Kern\;}
 \Prob(\Words)
 \xrightarrow{\;\Pi\;}
 \Delta^{K-1}
 \xrightarrow{\;\psi\;}
 \Delta^{K-1}.
\]
Here $\Kern_c$ is the observable continuation law, $\Pi$ groups prespecified
meaning-equivalent continuations, and $\psi$ is a finite-dimensional
calibration map estimated on data disjoint from evaluation.  The continuation
law is otherwise unrestricted, so the construction is semiparametric: the
language component is an unknown probability kernel, while only the
low-dimensional inverse calibration is parametrized.  The resulting
$B(c)=\psi\{\Pi(\Kern_c)\}$ is the state to which the descent and stochastic
stability results apply.

This construction separates three questions that should not be conflated.
First, \emph{measurement completeness} asks whether the selected
continuations retain enough probability mass.  Second, \emph{statistical
identification} asks whether distinct reference states induce distinguishable
values of the grouped language law on the operating domain.  Third,
\emph{dynamic closure} asks whether the calibrated state retains enough
information to predict the effect of subsequent transformations.  Accurate
one-time calibration does not imply the third property.

For clarity, suppose the grouped observable is $q(c)$ and the calibration
model is $\psi_\theta(q)$.  Conditional identification on a compact operating
set $\mathcal Q$ requires
\[
 \psi_{\theta_1}(q)=\psi_{\theta_2}(q)
 \quad\text{for all }q\in\mathcal Q
 \quad\Longrightarrow\quad
 \theta_1=\theta_2,
\]
or the corresponding local full-rank condition for a differentiable model.
Stable recovery additionally requires the inverse modulus not to approach
zero.  These are assumptions about the declared measurement design, not
properties granted by fluency or by a model's printed numerical confidence.
The completed experiments accordingly test retained mass, a lower
identification modulus, perturbation error, and held-out sequential coverage.

The main paper does not require a universal estimator for $\psi$.  It requires
only a frozen estimator whose error can enter Theorem~\ref{thm:push} and whose
induced closure defects enter Theorem~\ref{thm:stochastic-perturbation}.
Full asymptotic theory for nonlinear semiparametric recovery is therefore
supporting material rather than part of the foundational theorem chain; it is
developed in \cite{dixon2026semanticobservation}.  The measurement design and
its held-out certification are developed separately in
\cite{dixon2026}.

\section{Boundary with adjacent theories}

The construction uses established ideas at three levels but addresses a
different junction between them.  Future-behaviour quotients, automata, and
coalgebra explain how a minimal state can be constructed from observable
responses \cite{jacobs2023,rutten2000}.  Random dynamical systems and iterated
random functions establish existence and synchronization once a state and its
random update maps are given \cite{arnold1998,diaconis1999}.  Semantic
uncertainty methods group or calibrate language probabilities into
application-relevant meanings \cite{farquhar2024,kuhn2023}.

The missing link is whether a prompt-dependent language law descends to a
state on which stochastic evolution is well defined.  The contribution here
is to make that link explicit: semantic descent supplies the admissible state;
its fibre defect quantifies irreducible nonclosure; and random-iteration
theory supplies the stochastic recursion only after that defect has been
controlled.  The term \emph{stochastic lexical calculus} names this combined
construction, not a replacement for probability theory, stochastic calculus,
or existing linguistic calculi.

\end{document}